%% file: main.tex
\documentclass[11pt,letterpaper]{article}

\usepackage[letterpaper,margin=1in]{geometry}
\usepackage{times}
\usepackage[round,authoryear]{natbib}

\usepackage[hypertexnames=false]{hyperref}
\usepackage{url}
\usepackage{booktabs}
\usepackage{graphicx}
\usepackage{placeins}
\usepackage{algorithm}
\usepackage{algpseudocode}
\usepackage{amsmath,amssymb,amsthm,mathtools}
\usepackage{microtype}
\newcommand{\R}{\mathbb{R}}
\newcommand{\N}{\mathbb{N}}
\newcommand{\E}{\mathbb{E}}
\newcommand{\Pcal}{\mathcal{P}}
\newcommand{\Kcal}{\mathcal{K}}
\newcommand{\Ucal}{\mathcal{U}}
\newcommand{\Xcal}{\mathcal{X}}
\newcommand{\Ycal}{\mathcal{Y}}
\newcommand{\Gcal}{\mathcal{G}}
\newcommand{\Ncal}{\mathcal{N}}

\newcommand{\Law}{\mathcal{L}}
\newcommand{\supp}{\operatorname{supp}}
\newcommand{\dist}{\operatorname{dist}}
\newcommand{\diam}{\operatorname{diam}}
\newcommand{\diag}{\operatorname{diag}}
\newcommand{\tr}{\operatorname{tr}}
\newcommand{\Lip}{\operatorname{Lip}}
\newcommand{\SW}{\operatorname{SW}}

\newcommand{\Unif}{\operatorname{Unif}}
\newcommand{\softmax}{\operatorname{softmax}}

\newcommand{\pushforward}[2]{{#1}_{\#}#2}

\newtheorem{theorem}{Theorem}[section]
\newtheorem{lemma}[theorem]{Lemma}
\newtheorem{proposition}[theorem]{Proposition}
\newtheorem{corollary}[theorem]{Corollary}

\hypersetup{
  pdftitle={From Distributions to Stochastic Processes: Neural Approximation of Measure-Valued Maps},
  pdfauthor={Yichen Wang, Ziyi Wang, Wenlian Lu, Chenghuang Shen, Jianfeng Liu, Zhengdong Xiao, Longjiu Luo, Qianrong Wang},
  pdfsubject={Distribution-to-distribution learning and neural approximation of measure-valued maps}
}

\title{\texorpdfstring{From Distributions to Stochastic Processes: \\Neural Approximation of Measure-Valued Maps}
{From Distributions to Stochastic Processes: Neural Approximation of Measure-Valued Maps}}

\author{
Yichen Wang\textsuperscript{3},
Ziyi Wang\textsuperscript{4},
Wenlian Lu\textsuperscript{5},
Chenghuang Shen\textsuperscript{2,3},\\
Jianfeng Liu\textsuperscript{1},
Zhengdong Xiao\textsuperscript{1},
Longjiu Luo\textsuperscript{1},
Qianrong Wang\textsuperscript{1}\\
\textsuperscript{1}Alibaba Group\\
\textsuperscript{2}Alibaba Group, Hangzhou, China\\
\textsuperscript{3}Shanghai Center for Mathematical Sciences, Shanghai, China\\
\textsuperscript{4}Department of Mathematics, Fudan University, Shanghai, China\\
\textsuperscript{5}Center for Applied Mathematics, Fudan University, Shanghai, China\\
\texttt{yichenwang23@m.fudan.edu.cn, 21110180058@m.fudan.edu.cn}\\
\texttt{wenlian@fudan.edu.cn, austinshen9618@hotmail.com}\\
\texttt{ljf7451320@hotmail.com, tianqiao.xzd@alibaba-inc.com}\\
\texttt{luolongjiu@gmail.com, qianrong.wqr@alibaba-inc.com}
}

\begin{document}

\maketitle

\begin{abstract}
Learning mappings between probability distributions arises naturally 
in settings where both inputs and outputs are represented by populations 
of samples rather than individual observations. 
This paper develops an approximation-theoretic framework for such 
distribution-to-distribution learning problems and extends it to mappings between 
stochastic processes. For continuous operators acting on $W_2$ -compact families 
of finite-dimensional input laws, we establish uniform neural approximation in the
2-Wasserstein metric. Our construction relies on finite law statistics as input 
representations, a simplex-valued neural map, and a shared atomic output support 
that guarantees the output is always a valid probability measure. We then extend 
this approximation principle to probability laws on separable Hilbert spaces via 
fixed finite-rank orthogonal projections. 
Together, these results establish the 
representational feasibility of learning transformations whose inputs and outputs 
are probability laws rather than deterministic vectors or functions.

To assess practical relevance, we construct numerical problems whose target transformations 
are naturally defined at the level of distributions or probability laws: predicting 
the first-passage-time distribution of an Ornstein--Uhlenbeck process and the nonlinear 
response-path laws of a Duffing oscillator. Since the theoretical framework is 
deliberately model-agnostic and covers a substantially broader class of operators 
than can be efficiently represented by any single practical architecture, our experiments 
instantiate the framework using task-adapted neural models rather than reproducing 
the theoretical construction verbatim. Both models outperform a fixed-feature MLP 
baseline and distribution-space kernel regression. These proof-of-concept demonstrations 
complement the approximation theory by supporting the learnability and practical relevance 
of distribution-to-distribution transformations for random systems.
\end{abstract}

\section{Introduction}
\label{sec:introduction}

\paragraph{Motivation and problem setup.}
Many random systems are naturally described at the population level. In
uncertainty propagation, mean-field dynamics, ensemble forecasting, stochastic
PDEs, and Bayesian inverse problems, the desired prediction is an output law
rather than the response of one prescribed input
realization~\citep{lord2014spde,walsh1986spde,sznitman1991chaos,
carmona2018meanfield,stuart2010inverse,evensen2003enkf}. When the observations
are paths or fields, the same question concerns laws on function spaces. In
both cases, probability distributions are the mathematical objects of
interest, while finite ensembles are observations of those objects.

We formulate this task as learning a \emph{distributional operator}
\begin{equation*}
\Gcal:\Kcal\subset\Pcal_2(\Xcal)\longrightarrow\Pcal_2(\Ycal),
\qquad \mu\longmapsto\Gcal(\mu),
\end{equation*}
where $\Pcal_2$ denotes probability measures with finite second moment.
Training observations are pairs of empirical laws
$\{(\widehat\mu_i,\widehat\nu_i)\}_{i=1}^n$. The input and output ensembles
within a pair may be sampled independently, so there need not be a meaningful
map between individual particles. The target is instead the law-level relation
$\nu_i=\Gcal(\mu_i)$. We measure approximation in the 2-Wasserstein metric
$W_2$. The finite-dimensional setting takes $\Xcal=\R^{d_x}$ and
$\Ycal=\R^{d_y}$; the process setting takes $\Xcal$ and $\Ycal$ to be
separable Hilbert spaces of paths or fields.

\paragraph{Related work and the theoretical gap.}
Classical universal approximation theory treats continuous maps between
finite-dimensional vectors~\citep{cybenko1989approximation}. Neural-operator
theory extends the encoding--mapping--reconstruction principle to
deterministic functions, with DeepONet and Fourier neural operators providing
prominent examples~\citep{lu2021deeponet,li2021fno,
kovachki2024neuraloperator}. Our input and output, however, are probability
laws. A deterministic operator between functions does not establish that an
input law can be encoded by finitely many observable statistics, transformed
continuously, and decoded into a valid output law with uniform error control.

Distribution regression and invariant set networks address complementary
parts of the problem. Distribution regression maps an empirical law to a
scalar or vector, while distribution-to-distribution regression permits a
distributional response~\citep{szabo2016distribution,
oliva2013distribution}. DeepSets and Set Transformers provide practical
permutation-invariant encoders for unordered samples~\citep{zaheer2017deepsets,
lee2019settransformer}, but an architecture for finite sets alone does not
give a uniform approximation theorem on a space of laws. More specifically,
\citet{oliva2013distribution} assume densities on compact Euclidean domains
and derive an expected $L_2$ density-risk bound for a kernel-smoothed
orthogonal-series estimator. Their setting does not yield uniform $W_2$
approximation or guarantee that a truncated output series is a probability
density. \citet{pham2023meanfieldnn} approximate maps of the form
$(x,\mu)\mapsto V(x,\mu)$, whose outputs are vectors or functions rather than
probability laws.

Measure-valued approximation is closer to our objective. The probabilistic
transformer maps finite-dimensional features to mixtures of fixed measures
and proves compact-uniform approximation for $\Pcal_1(\R^D)$-valued maps in
$W_1$~\citep{kratsios2021universal}; the general framework of
\citet{kratsios2023metric} covers maps between broad classes of metric spaces,
including probability-valued targets. Transformer dynamics can also
interpolate finitely many prescribed input--target measure
pairs~\citep{geshkovski2024measure}. These results establish important general
measure-valued approximation principles. What remains unresolved for our
setting is their combination with a law-valued input accessed through an
empirical ensemble: a finite encoder that captures $W_2$ rather than only weak
behavior, a decoder that always returns a probability measure, a uniform bound
for replacing feature integrals by sample averages, and a lifting from
finite-dimensional laws to laws on path spaces. The gap is particularly
delicate because $W_2$ convergence requires both weak convergence and control
of second moments, and path-space reduction adds input- and output-projection
errors.

\paragraph{Our theory.}
For a $W_2$-compact family
$\Kcal\subset\Pcal_2(\R^{d_x})$ and a $W_2$-continuous operator
$\Gcal:\Kcal\to\Pcal_2(\R^{d_y})$, Theorem~\ref{thm:neural-realization}
constructs a finite law encoder from bounded test-function expectations and a
truncated second-moment feature. A neural map with a softmax output assigns
weights to a common finite set of output atoms. The bounded tests control weak
behavior, the moment feature supplies the additional information required by
$W_2$, and the simplex-valued decoder guarantees nonnegative weights with unit
mass. Compactness and continuity then give uniform $W_2$ approximation of the
target operator. For empirical inputs, Proposition~\ref{prop:empirical}
replaces population feature integrals by independent sample averages and
separates the population approximation error from a uniform high-probability
query-sampling error.

Theorem~\ref{thm:path-space} extends the result to laws on separable Hilbert
spaces. Fixed finite-rank orthogonal maps encode paths into coordinate vectors,
the finite-dimensional distributional operator acts on the induced coordinate
laws, and synthesis maps reconstruct output paths. Uniform projection on
$W_2$-compact law families decomposes the total error into input projection,
finite-dimensional neural approximation, and output reconstruction. The
argument remains entirely at the level of probability laws and therefore does
not require eventwise pairing of input and output paths. Together, the two
theorems provide the missing $W_2$-specific link between empirical law
encoding, valid probability-valued decoding, and process-law approximation.
They establish representational feasibility rather than statistical training
rates or a mandatory practical architecture.

\paragraph{Contributions.}
Our contributions are:
\begin{itemize}
\item We formulate distribution-to-distribution learning from finite,
unpaired input and output ensembles as approximation of a continuous operator
between probability-law spaces.
\item We prove uniform neural approximation in $W_2$ on compact families of
finite-dimensional input laws using finitely many law statistics, neural
simplex weights, and a valid common-support output law. We also give a uniform
query-sample bound for empirical feature averages.
\item We lift the approximation result to process and random-field laws on
separable Hilbert spaces through fixed finite-rank projections, with explicit
control of input projection, finite-dimensional approximation, and output
reconstruction.
\item We demonstrate the framework on a censored Ornstein--Uhlenbeck
first-passage law and nonlinear Duffing response-path laws. Task-specific
categorical and generative decoders assess practical learnability against
fixed-feature and distribution-space kernel baselines without identifying the
existence construction with a unique architecture.
\end{itemize}

\section{Universal Approximation Theory for Operators between Probability Laws}
\label{sec:chapter3}

We establish two universal approximation results: one for operators between
finite-dimensional probability laws and one for laws on separable Hilbert
spaces. Both use encoding, neural mapping, and valid probability-law
reconstruction. Supporting constructions, empirical-feature bounds, and
projection lemmas are summarized below; their formal statements and all
proofs appear in Appendix~\ref{app:proofs}.

\subsection{Problem setup and assumptions}
\label{subsec:problem-setup}

We begin with the standard probability-space notation used throughout.  If
$(\Xcal,d_\Xcal)$ is a Polish space, $\Pcal(\Xcal)$ denotes its Borel
probability measures and, for $1\leq p<\infty$,
\begin{equation}
 \Pcal_p(\Xcal)=\left\{\mu\in\Pcal(\Xcal):
 \int_\Xcal d_\Xcal(x,x_0)^p\,d\mu(x)<\infty\right\},
 \label{eq:p-moment-space}
\end{equation}
where the definition does not depend on the fixed point $x_0\in\Xcal$.  For
$\mu,\nu\in\Pcal_p(\Xcal)$, let $\Pi(\mu,\nu)$ be their set of couplings and
define
\begin{equation}
 W_p^\Xcal(\mu,\nu)=
 \inf_{\pi\in\Pi(\mu,\nu)}
 \left(\int_{\Xcal\times\Xcal}d_\Xcal(x,y)^p\,d\pi(x,y)\right)^{1/p}.
 \label{eq:wasserstein-definition}
\end{equation}
When the underlying space is clear, we simply write $W_p$.  For a measurable
map $F:\Xcal\to\Ycal$, $F_\#\mu$ denotes the pushforward law, and for
i.i.d. samples $X_1,\ldots,X_{N_{\mathrm{in}}}\sim\mu$ we write
\begin{equation}
 \widehat\mu_{N_{\mathrm{in}}}
 =\frac1{N_{\mathrm{in}}}\sum_{b=1}^{N_{\mathrm{in}}}\delta_{X_b}
 \label{eq:empirical-measure}
\end{equation}
for the empirical measure.

Let $\Xcal$ and $\Ycal$ be Polish spaces.  A distributional operator is a map
\begin{equation*}
 \Gcal:\Kcal\subset\Pcal_2(\Xcal)\longrightarrow\Pcal_2(\Ycal).
\end{equation*}
Its uniform approximation error on $\Kcal$ is measured by
$\sup_{\mu\in\Kcal}W_2^\Ycal(\Ncal(\mu),\Gcal(\mu))$.

The finite-dimensional theory takes $\Xcal=\R^{d_x}$ and
$\Ycal=\R^{d_y}$ and uses the following assumptions throughout:
\begin{enumerate}
 \item[(A1)] $\Kcal\subset\Pcal_2(\R^{d_x})$ is $W_2$-compact.
 Consequently, the family is tight and its second moments are uniformly
 integrable:
 \begin{equation}
  \lim_{R\to\infty}\sup_{\mu\in\Kcal}
  \int_{\{\|x\|_2>R\}}\|x\|_2^2\,d\mu(x)=0.
  \label{eq:ui-second-moment}
 \end{equation}
 \item[(A2)] $\Gcal:\Kcal\to\Pcal_2(\R^{d_y})$ is $W_2$-continuous.
 Hence $\Gcal$ is uniformly continuous on $\Kcal$ and $\Gcal(\Kcal)$ is
 $W_2$-compact.
 \item[(A3)] The chosen continuous, nonpolynomial activation gives uniform
 neural-network approximation on compact Euclidean sets, as in the classical
 universal approximation theorem.  Bounded clipping, or an equivalent
 bounded output activation, is available for the feature networks.
\end{enumerate}

Two kinds of input statistics are needed.  Choose bounded
weak-convergence tests $\psi_1,\ldots,\psi_M\in C_b(\R^{d_x})$ and retain the
second moment as a separate coordinate:
\begin{align}
 E_{\rm wk}(\mu)&=\left(\int\psi_1\,d\mu,\ldots,
 \int\psi_M\,d\mu\right),\nonumber\\
 m_2(\mu)&=\int\|x\|_2^2\,d\mu(x),\nonumber\\
 E(\mu)&=(E_{\rm wk}(\mu),m_2(\mu))\in\R^{M+1}.
 \label{eq:test-encoder}
\end{align}
The bounded coordinates determine weak behavior, while convergence of the
second moments upgrades weak convergence to $W_2$ convergence.  Compactness
then permits a finite choice of tests at any prescribed accuracy.

The construction has three layers.  First, $E$ represents the input law to the
accuracy required by the uniform continuity of $\Gcal$.  Second, a continuous
simplex-valued map assigns weights to a common finite output support.  Third,
ordinary neural networks approximate the feature functions and simplex map.
The Tietze--Dugundji extension theorem supplies the simplex-valued
extension needed in the second layer, while tightness and
\eqref{eq:ui-second-moment} control neural approximation of the feature
integrals.

\subsection{Universal approximation of finite-dimensional laws}
\label{subsec:finite-laws}

The first theorem combines finite law representation with neural realization.
The encoder integrates bounded neural features, including a truncated
approximation of the second moment, and a softmax network predicts weights
on a common finite output support.

\begin{theorem}[Universal approximation of distributional operators]
\label{thm:neural-realization}
Assume \textup{(A1)--(A3)}. Let
$\Delta_N=\{w\in[0,1]^N:\sum_{j=1}^Nw_j=1\}$.  For every $\varepsilon>0$, there exist $M,N\in\N$, common output atoms
$a_1,\ldots,a_N\in\R^{d_y}$, bounded scalar feature networks $\widehat\psi_i$, a bounded neural approximation $\widehat q_R$ of
$q(x)=\|x\|_2^2$ obtained by truncation at some radius $R$, and a softmax
network $\widehat p:\R^{M+1}\to\Delta_N$.  Define
\begin{align}
 \widehat E_{{\rm wk},R}(\mu)
 &=\left(\int\widehat\psi_1\,d\mu,\ldots,
 \int\widehat\psi_M\,d\mu\right),\nonumber\\
 \widehat E_R(\mu)
 &=\left(\widehat E_{{\rm wk},R}(\mu),
 \int\widehat q_R\,d\mu\right),
 \label{eq:neural-realization}
\end{align}
and
\begin{equation}
 \widehat\Ncal_\theta(\mu)
 =\sum_{j=1}^N\widehat p_j(\widehat E_R(\mu))\delta_{a_j}.
 \label{eq:neural-finite-support}
\end{equation}
Then the networks can be chosen so that
\begin{equation*}
 \sup_{\mu\in\Kcal}
 W_2(\widehat\Ncal_\theta(\mu),\Gcal(\mu))<\varepsilon.
\end{equation*}
\end{theorem}

\paragraph{Construction and proof outline.}
Compactness permits finitely many bounded tests and a second-moment
coordinate to distinguish input laws to any required $W_2$ tolerance
(Lemma~\ref{lem:finite-separation}). Uniform continuity of $\Gcal$ transfers
this tolerance to the outputs. Local finite-support output approximations
are collected on common atoms, and a partition of unity gives continuous
simplex weights. The formal continuous construction is
Theorem~\ref{thm:finite-support} in Appendix~\ref{app:finite-construction}.
Bounded tests are then approximated by clipped networks on large balls;
tightness controls the complement, and uniform integrability controls
truncation of the quadratic feature (Lemma~\ref{lem:feature-approx}).
Regularizing the weights, approximating their logits, and applying softmax
preserves nonnegativity and unit mass. Common-support stability converts
weight error to $W_2$ error (Lemma~\ref{lem:common-support}). The full proof
of Theorem~\ref{thm:neural-realization} is in
Appendix~\ref{app:neural-proof}.

The fixed atoms are an existence device shared by all input laws. Practical
models may use richer structured decoders, but their approximation properties
require separate justification.

\paragraph{Access through empirical ensembles.}
For an already realized neural operator, the population feature integrals can
be replaced by averages of independent query samples. Hoeffding concentration
and the output network's modulus of continuity give a high-probability bound
uniform over the queried law in $\Kcal$. This separates population
approximation error from empirical feature-estimation error. With a Lipschitz
weight network, the latter contributes an output $W_2$ bound of order
$N_{\mathrm{in}}^{-1/4}$ for fixed confidence and realized architecture. The formal bound
and its constants are in Proposition~\ref{prop:empirical},
Appendix~\ref{app:empirical-bound}. This is a query-sample guarantee, not a
training, optimization, or network-width rate. Boundedness includes the
truncated quadratic feature; the same concentration conclusion need not
hold for an untruncated second moment.

\subsection{Universal approximation of Hilbert-space laws}
\label{subsec:path-laws}

For process and random-field laws, finite-dimensionalization acts on each
path and therefore produces a coordinate law, rather than a single feature
vector representing the whole input distribution. Let $\Xcal$ and $\Ycal$
be separable Hilbert spaces. Fix increasing finite-dimensional subspaces
$X_n\subset\Xcal$ and $Y_m\subset\Ycal$ with dense unions, with
$d_n=\dim X_n$ and $q_m=\dim Y_m$. Orthonormal coordinates define analysis
maps $\mathcal E_\Xcal^n:\Xcal\to\R^{d_n}$ and
$\mathcal E_\Ycal^m:\Ycal\to\R^{q_m}$ and synthesis maps
$\mathcal D_\Xcal^n:\R^{d_n}\to\Xcal$ and
$\mathcal D_\Ycal^m:\R^{q_m}\to\Ycal$. Their compositions
\begin{equation*}
P_\Xcal^n=\mathcal D_\Xcal^n\circ\mathcal E_\Xcal^n,
\qquad P_\Ycal^m=\mathcal D_\Ycal^m\circ\mathcal E_\Ycal^m
\end{equation*}
are orthogonal projections. Their pushforwards approximate every
$W_2$-compact law family uniformly; the formal reconstruction and stability
results are Lemmas~\ref{lem:hilbert-autoencoding} and~\ref{lem:pushforward}
in Appendix~\ref{app:hilbert-proof}.

\begin{theorem}[Universal approximation on separable Hilbert-space laws]
\label{thm:path-space}
Assume the activation condition \textup{(A3)}. Let $\Xcal$ and $\Ycal$
be separable Hilbert spaces equipped with the
orthogonal coordinate systems above.  Let
$\Kcal\subset\Pcal_2(\Xcal)$ be $W_2^\Xcal$-compact, and let
\begin{equation*}
 \Gcal:\Ucal\longrightarrow\Pcal_2(\Ycal)
\end{equation*}
be $W_2$-continuous on an open neighborhood $\Ucal$ of $\Kcal$.  For every
$\varepsilon>0$, there exist $n,m\in\N$ and a finite-dimensional neural
distributional operator
$\widehat\Ncal_\theta:\Pcal_2(\R^{d_n})\to\Pcal_2(\R^{q_m})$ of the form in
Theorem~\ref{thm:neural-realization} such that
\begin{equation}
 \Ncal_{\theta,n,m}
 = (\mathcal D_\Ycal^m)_\#\circ\widehat\Ncal_\theta
 \circ(\mathcal E_\Xcal^n)_\#
 \label{eq:path-law-theorem-construction}
\end{equation}
satisfies
\begin{equation*}
 \sup_{\mu\in\Kcal}W_2^\Ycal(
 \Ncal_{\theta,n,m}(\mu),\Gcal(\mu))<\varepsilon.
\end{equation*}
\end{theorem}

\paragraph{Proof outline and interpretation.}
Uniform projection accuracy first makes $(P_\Xcal^n)_\#\mu$ close to
$\mu$ over $\Kcal$. The open neighborhood ensures that projected input
laws remain in the domain of $\Gcal$, and continuity controls the induced
output perturbation. Uniform output projection reduces the target to a
finite-dimensional coordinate law. Theorem~\ref{thm:neural-realization}
approximates the resulting law operator, and synthesis is an isometry.
The triangle inequality thus separates input projection, finite-dimensional
neural approximation, and output reconstruction errors. The complete proof
is in Appendix~\ref{app:hilbert-proof}; further path-law construction details
are in Appendix~\ref{app:path-law-details}.

The result requires no eventwise coupling of input and output paths. Its
orthogonal coordinate systems are fixed independently of the input law;
a distribution-specific PCA or Karhunen--Lo\`eve basis, or a learned encoder,
requires additional stability and uniform-reconstruction analysis.

\section{Experiments}
\label{sec:experiments}

Theoretical results are typically established through general constructions
to ensure broad applicability, but such constructions are often not directly
implementable in practice. We therefore design concrete and computationally
feasible distributional neural operators for controlled finite-dimensional
Gaussian distributions, first-passage times of Ornstein--Uhlenbeck (OU)
processes, and nonlinear Duffing oscillators. Rather than mechanically
reproducing the theoretical
constructions, we follow their underlying principles and tailor each model to
the structure of its problem. The controlled Gaussian experiment acts on
finite-dimensional distributions, whereas the OU and Duffing tasks use
empirical laws of stochastic-process paths. Owing to space limitations, the main text focuses on
the OU and Duffing tasks; complete results for the controlled Gaussian task
are provided in Appendix~\ref{app:supplementary-experiments}, and the
corresponding theoretical extensions appear in
Appendix~\ref{app:structured-output-theory}.

\begin{figure}[t]
\centering
\includegraphics[width=\textwidth]{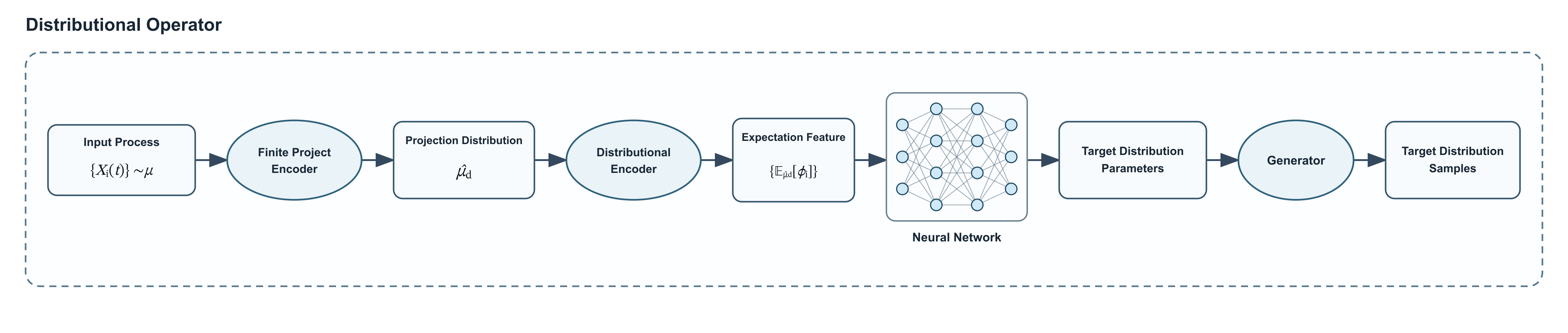}
\caption{A general example of the model pipeline. The model used for a
specific task may differ in certain details, but broadly follows this design
rationale. For task-specific model structures, see
Figures~\ref{fig:ou-distributional-operator-pipeline}
and~\ref{fig:duffing-process-operator-pipeline} in the appendix.}
\label{fig:distributional-operator-overview}
\end{figure}

To the best of our knowledge, these tasks have no established benchmark or
off-the-shelf baseline. We therefore compare against task-appropriate
fixed-feature and kernel-regression alternatives. Implementation
details for data generation and the comparison methods are provided in
Appendix~\ref{app:experiment-protocols}.

Unless otherwise specified, input and output ensembles are sampled
independently within each law pair, so the models cannot exploit
trajectory-level correspondence. Every experiment uses 1,000 training law
pairs and 200 test law pairs, which exhaust the 1,200-law dataset. The test
set is used only to evaluate final model performance. Neural experiments are
repeated with five random seeds. The default neural training configuration is
AdamW for 1,000 epochs with initial learning rate $10^{-3}$ and weight decay
$10^{-2}$. All computations are performed on one
NVIDIA GeForce RTX 4090 GPU with 24 GB of memory. Source code is available at
\url{https://anonymous.4open.science/r/distributional_operator-3E6C}.

\subsection{Experiment I: OU First-Passage Distribution}
\label{sec:ou-first-passage}

\subsubsection{Problem}

The OU first-passage problem has a standard interpretation as a
leaky integrate-and-fire (LIF) neuron. In this model, the membrane potential
evolves until it first reaches a firing threshold, at which point a spike is
recorded and the potential is reset~\citep{burkitt2006lif}. Under the classical
diffusion approximation with Gaussian white-noise input, the subthreshold
membrane potential is an OU process. Consequently, the elapsed time from a
reset to the next threshold crossing is one inter-spike interval (ISI), whose
law is the OU first-passage-time law~\citep{ricciardi1988ou,alili2005ou}.
Specifically, the membrane potential follows
\begin{equation}
dV_t=\left(-\frac{V_t}{\gamma}+m\right)dt+\sqrt q\,dW_t,
\qquad V(0)=v_0,
\label{eq:ou-process}
\end{equation}
with threshold $V_{\mathrm{th}}$ and first-passage time
\begin{equation}
T=\inf\{t\geq0:V(0)=v_0, V_t\geq V_{\mathrm{th}}\}.
\end{equation}
The model does not observe $(m,q)$ directly. Instead, the input law is that of
the drifted Brownian process
\begin{equation}
dI_t=m\,dt+\sqrt q\,dW_t,
\qquad I(0)=0,
\label{eq:ou-input-process}
\end{equation}
with the same $(m,q)$. The target distributional operator is
\begin{equation*}
\Gcal:\Law(I)\longmapsto\Law(T),
\end{equation*}
where the first-passage law is represented by 48 finite-time categories and
one censoring category. The dataset contains 1,200 input--output law pairs.
Each input law is represented by $N_{\mathrm{in}}$ paths on 256 observation
times, and each target law by $N_{\mathrm{out}}$ independently simulated
reset trials. The ensembles share the same law parameters but have no
samplewise pairing. The parameter family, simulation scheme, censoring rule,
binning, and split are detailed in Appendix~\ref{app:ou-data}.

\subsubsection{Model Instantiation}

Following the process pipeline, we use a training-fitted PCA encoder
$\mathcal E_I$ to preprocess the paths, selecting the truncation dimension
$d_X$ so that the cumulative explained-variance ratio exceeds $0.99$. For the
$i$th input ensemble, the Distributional Operator computes
\begin{align*}
\alpha_{ij}&=\mathcal E_I(I_{ij})\in\R^{d_X},\\
f_{ij}&=\phi_{\mathrm{path}}(\alpha_{ij}),\\
h_i&=\frac1{N_{\mathrm{in}}}\sum_{j=1}^{N_{\mathrm{in}}}
\phi_{\mathrm{set}}(f_{ij}),\\
\ell_i&=\rho_\theta(h_i),\qquad
\widehat p_i=\softmax(\ell_i)\in\Delta_{49}.
\end{align*}
The shared path and set encoders followed by mean aggregation make the model
invariant to the ordering of the input paths. The first 48 entries of
$\widehat p_i$
parameterize the finite first-passage intervals and the last entry represents
censoring. Training minimizes the count-weighted categorical negative
log-likelihood. PCA fitting, layer widths, parameter counts, optimization,
checkpointing, and complete baseline configurations are reported in
Appendix~\ref{app:ou-baselines}.

\subsubsection{Evaluation Metrics and Results}

Negative log-likelihood is the categorical training objective and measures
predictive fit. Hellinger distance and KL divergence compare the full
49-category distributions from complementary symmetric and directed
perspectives. Finite-bin $W_2$ measures displacement of probability mass over
passage time after removing and renormalizing the censoring category, while
tail error measures the absolute error in the probability of no crossing.
Lower is better for every metric; formal definitions are in
Appendix~\ref{app:ou-metrics}.

\begin{table}[htbp]
\centering
\small
\setlength{\tabcolsep}{3.5pt}
\resizebox{\textwidth}{!}{%
\begin{tabular}{lrrrrr}
\toprule
Model & NLL & Hellinger & KL divergence & $W_{2,\mathrm{fin}}$ & Tail error\\
\midrule
Distributional Operator & $3.11176 \pm 0.00085$ & $0.18055 \pm 0.00048$ & $0.10636 \pm 0.00085$ & $0.16634 \pm 0.00286$ & $0.00373 \pm 0.00018$\\
Fixed-Feature MLP & $3.11680 \pm 0.00011$ & $0.18592 \pm 0.00013$ & $0.11140 \pm 0.00011$ & $0.20957 \pm 0.00026$ & $0.00399 \pm 0.00003$\\
Kernel regression & $3.15241$ & $0.21616$ & $0.14701$ & $0.34567$ & $0.00639$\\
\bottomrule
\end{tabular}%
}
\caption{Test results for input process laws mapped to OU first-passage-time
laws. Entries are the mean $\pm$ standard deviation over five runs with random
initializations. Kernel regression is deterministic and therefore has
zero standard deviation.  Lower is better for every metric.}
\label{tab:ou}
\end{table}

For a qualitative theoretical reference, we evaluate the OU first-passage-time
density using its Bessel-bridge representation~\citep{alili2005ou}. This
reference density is independent of the empirical target histograms used for
training and lets us compare the shapes predicted by the learned and
kernel-based models against a continuous first-passage-time calculation.

\begin{figure}[t]
\centering
\includegraphics[width=\textwidth]{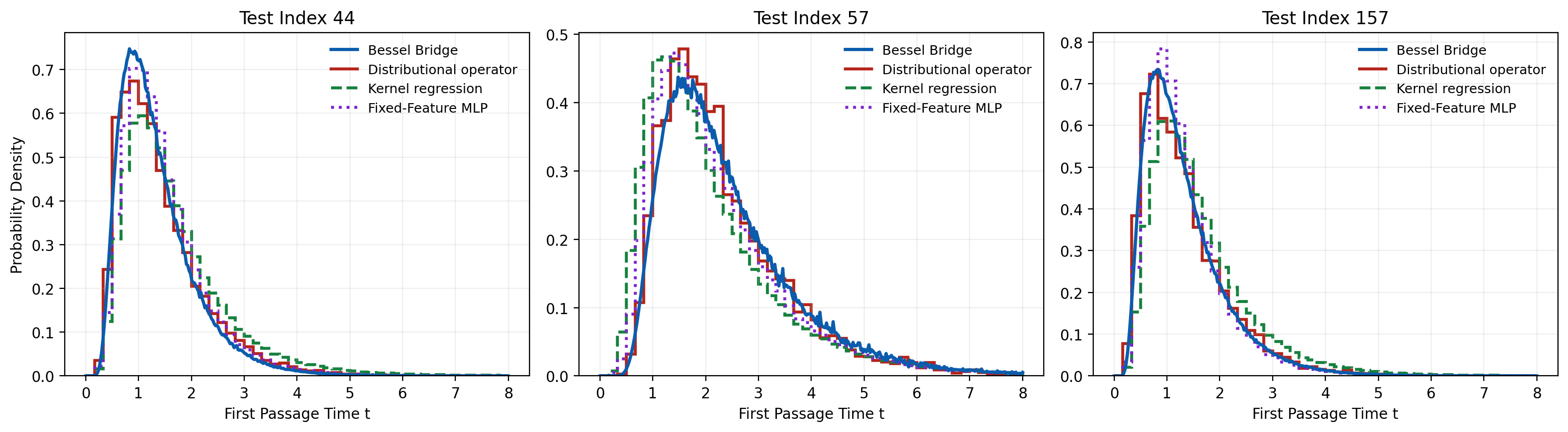}
\caption{OU first-passage-time probability densities for three representative
test laws. The continuous blue curve is the theoretical reference calculated
from the Bessel-bridge representation. The step curves show the densities
predicted by the Distributional Operator, kernel regression, and the
Fixed-Feature MLP on the 48 finite time bins. Agreement with the reference
curve compares how well the models recover the location, height, and tail of
the first-passage-time density.}
\label{fig:ou-out-examples}
\end{figure}

The Distributional Operator gives the lowest mean NLL, Hellinger
distance, KL divergence, finite-bin $W_2$, and tail error. Kernel regression
is worse on every metric, indicating that local
averaging between training laws does not capture the map as accurately. The
density comparisons in Figure~\ref{fig:ou-out-examples} show that the learned
operator recovers the location, peak, and tail of the theoretical
first-passage-time density across representative test laws.

\subsection{Experiment II: Duffing Uncertainty Propagation}
\label{sec:duffing}

\subsubsection{Problem}

We study a distribution-to-distribution map induced by the fixed hardening
Duffing oscillator
\begin{equation*}
\ddot Y(t)+0.2\dot Y(t)+Y(t)+Y(t)^3=X(t),
\qquad Y(0)=\dot Y(0)=0,
\end{equation*}
on $t\in[0,20]$. Given a probability law $\mu$ of stochastic forcing paths,
the target is the response law under the Duffing solution operator
$\mathcal S$, namely
\begin{equation*}
\Gcal:\mu\longmapsto\mathcal S_{\#}\mu.
\end{equation*}
This is a distributional uncertainty-propagation problem: the model must
characterize an ensemble of nonlinear responses rather than predict the
response to one prescribed forcing trajectory. Such laws arise in nonlinear
vibration and energy-harvesting models under stochastic or broadband
excitation~\citep{daqaq2010duffing,jia2020energy}. The dataset contains 1,200
input--output law pairs. Each law is represented by $N_{\mathrm{in}}$ forcing
paths and $N_{\mathrm{out}}$ independently sampled response paths on 256
observation times. The ensembles share the same forcing-law parameters but
have no samplewise pairing. The forcing family, numerical solver, and
measure-level construction are detailed in Appendix~\ref{app:duffing-data}.

\subsubsection{Model Instantiation}

Following the process pipeline, we use training-fitted PCA encoders
$\mathcal E_X$ and $\mathcal E_Y$ to preprocess the forcing and response
paths. Their truncation dimensions $d_X$ and $d_Y$ are selected so that each
cumulative explained-variance ratio exceeds $0.99$. For the $i$th input
ensemble,
\begin{align*}
\alpha_{ij}&=\mathcal E_X(X_{ij})\in\R^{d_X},\\
c_i&=\rho_\theta\left(
\frac1{N_{\mathrm{in}}}\sum_{j=1}^{N_{\mathrm{in}}}
\phi_\theta(\alpha_{ij})\right),\\
z_{ir}&\sim\mathcal N(0,I_{d_Y}),\qquad
\beta_{ir}=F_\theta(z_{ir};c_i),\\
\widehat Y_{ir}&=\mathcal D_Y(\beta_{ir}),\qquad
r=1,\ldots,N_{\mathrm{out}},
\end{align*}
where $F_\theta$ is a conditional RealNVP and $\mathcal D_Y$ is the inverse
output-PCA reconstruction map~\citep{dinh2017realnvp}. Thus the empirical
forcing-process law is projected, encoded by expectation features, combined
with independent latent draws, and decoded into samples from the predicted
response-process law.

The conditional generator uses six affine coupling transformations. For a
masked partition $h=(h_{\mathcal I},h_{\mathcal J})$, each transformation has
the form
\begin{align*}
h'_{\mathcal I}&=h_{\mathcal I},\\
h'_{\mathcal J}&=h_{\mathcal J}\odot
\exp(\widetilde s_\ell(h_{\mathcal I},c_i))
+t_\ell(h_{\mathcal I},c_i),
\qquad \widetilde s_\ell=2\tanh(s_\ell/2).
\end{align*}
Training minimizes the debiased Sinkhorn divergence between
$N_{\mathrm{out}}$ generated and $N_{\mathrm{out}}$ target response paths.
Layer widths, parameter counts, PCA
fitting, coupling masks, optimization schedules, and complete baseline
configurations are reported in Appendices~\ref{app:duffing-model}
and~\ref{app:duffing-baselines}.

\subsubsection{Evaluation Metrics and Results}

Sinkhorn divergence is the training objective and
measures regularized transport discrepancy between response ensembles. RBF
MMD compares kernel mean embeddings and is sensitive to broad distributional
differences. Sliced $W_2$ averages one-dimensional transport discrepancies
over random path-space projections, while energy distance compares within-
and between-ensemble Euclidean distances. Formal definitions and evaluation settings
are in Appendix~\ref{app:duffing-metrics}. We compared the performance of distributional operator (proposed model)
 with that of two baseline models on the test data in Table~\ref{tab:duffing}.Lower values indicate closer
agreement for all four metrics. To mitigate the effects of randomness, the evaluation results 
 were computed using five different random seeds and are reported as the mean ± standard deviation of the returns, 
 except for the kernel regression model, which is deterministic and therefore reported as a single value. 
 To intuitively illustrate the predictive performance, we present \ref{fig:duffing-results}the results of the models on three randomly selected test datasets.

\begin{table}[htbp]
\centering
\scriptsize
\begin{tabular}{lcccc}
\toprule
Model & Sinkhorn & MMD & $\SW_2$ & Energy\\
\midrule
Distributional Operator 
& $8.2050\pm0.2786$
& $0.0760\pm0.0010$
& $0.1639\pm0.0036$
& $0.3407\pm0.0104$\\
Fixed-Feature MLP
& $9.9426\pm0.6795$
& $0.0777\pm0.0017$
& $0.1718\pm0.0036$
& $0.3890\pm0.0173$\\
Kernel regression
& $18.6291$
& $0.1255$
& $0.2498$
& $0.8650$\\
\bottomrule
\end{tabular}
\caption{Test performance on the Duffing dataset. Neural-model entries are the
mean $\pm$ sample standard deviation on the 200-law test set. Kernel regression
is a single fit with a fixed bandwidth. Sinkhorn denotes the Sinkhorn divergence, MMD
the maximum mean discrepancy, $\SW_2$ the sliced 2-Wasserstein distance, and
Energy the energy distance; lower values indicate better agreement.}
\label{tab:duffing}
\end{table}

\begin{figure}[t]
\centering
\includegraphics[width=0.95\textwidth]{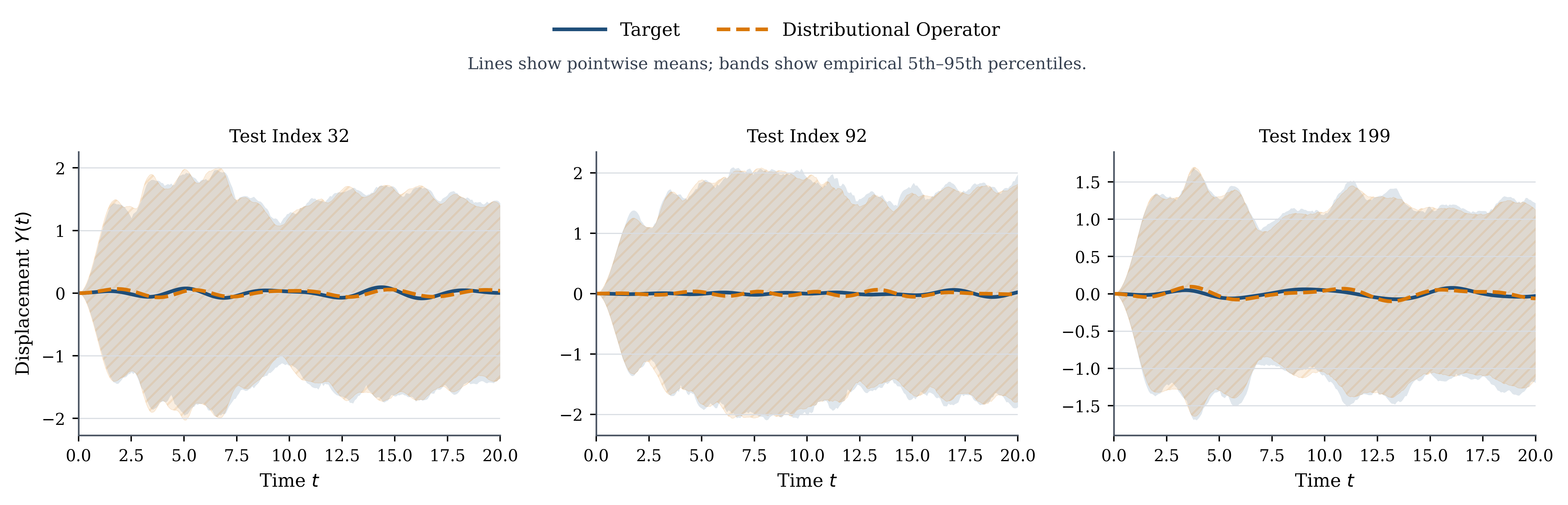}
\caption{Conditional distribution prediction for the Duffing oscillator. 
Target and predicted response distributions are shown for three randomly chosen test forcing measures. 
Solid blue lines and shaded regions denote the reference mean and 5th--95th percentile envelope, 
while dashed orange lines and hatched regions show the corresponding model predictions. 
Agreement in both statistics demonstrates the model's accuracy in approximating the nonlinear pushforward 
from forcing to response distributions.}
\label{fig:duffing-results}
\end{figure}

\FloatBarrier

\section{Conclusion}

We developed a constructive approximation theory for continuous
distribution-to-distribution maps in the 2-Wasserstein metric. On
$W_2$-compact input families, finitely many law statistics, a neural
simplex-valued map, and a shared atomic support suffice for uniform
approximation by valid output measures. We also quantified the error incurred
when those statistics are estimated from a finite unordered ensemble. By
combining finite-rank analysis and synthesis maps with the finite-dimensional
construction, the same principle extends to probability laws on separable
Hilbert spaces and therefore to distributional operators acting on stochastic
processes.

The OU first-passage and Duffing uncertainty-propagation experiments show that
this law-level formulation is computationally meaningful and learnable from
unpaired ensembles. The task-adapted categorical and conditional-flow models
outperform the fixed-feature and kernel-regression comparators on the reported
metrics. These architectures instantiate the framework's central principles
but do not reproduce the proof's finite-support construction, and the reported
comparisons are not evidence of architectural optimality: the baseline set is
necessarily limited, stronger designs may exist, and the remaining predictive
errors leave substantial room for improvement. Future work should tighten the
connection between constructive theory and practical architectures, weaken the
compactness and neighborhood assumptions, obtain sharper statistical and
optimization rates, and evaluate the framework on broader and more demanding
stochastic systems.

\bibliographystyle{iclr2026_conference}
\bibliography{references}

\section*{AI Use Disclosure}
Generative AI tools were used to polish the writing and improve readability;
to search for and retrieve relevant literature and technical information; to
prepare initial drafts of portions of the manuscript; and to assist in
generating the synthetic datasets used in the experiments. The authors
independently checked the retrieved sources and reviewed, revised, and
validated all AI-assisted text, citations, data-generation procedures, and
resulting datasets. The authors take responsibility for the final content of
this work, including all text, claims, and artifacts produced with the aid of
generative AI.

\appendix
\section{Table of Notations}
\label{app:notations}

\begin{center}
\centering
\small
\renewcommand{\arraystretch}{1.12}
\begin{tabular}{p{0.18\textwidth}p{0.75\textwidth}}
\toprule
Notation & Meaning \\
\midrule
$\R$, $\N$ & Real numbers and positive integers. \\
$\Xcal$, $\Ycal$ & Input and output state spaces; in the process setting these are separable Hilbert spaces. \\
$\Pcal_2(\Xcal)$ & Borel probability measures on $\Xcal$ with finite second moment. \\
$W_2$, $W_2^{\Xcal}$ & Quadratic Wasserstein distance, with the superscript indicating the underlying state space when needed. \\
$\Kcal$ & A $W_2$-compact family of admissible input laws. \\
$\mu$, $\nu$ & Generic input and output probability laws. \\
$\Gcal$ & Target distributional operator, mapping an input law to an output law. \\
$\Ncal$, $\widehat\Ncal_\theta$ & Continuous finite-support approximation and its neural realization. \\
$\delta_x$ & Dirac probability measure concentrated at $x$. \\
$F_\#$, $\pushforward{F}{\mu}$ & Pushforward of a measure through a measurable map $F$. \\
$E(\mu)$, $\widehat E(\mu)$ & Exact finite-dimensional law encoder and its neural approximation. \\
$\psi_i$, $\widehat\psi_i$ & Bounded test functions and their neural approximations. \\
$m_2(\mu)$ & Second moment $\int \lVert x\rVert_2^2\,d\mu(x)$. \\
$a_j$, $p_j$ & Shared output-support points and their simplex-valued weights. \\
$\Delta_N$ & Probability simplex $\{q\in[0,1]^N:\sum_{j=1}^N q_j=1\}$. \\
$M$, $N$ & Number of input features and output support points in the theoretical construction. \\
$N_{\mathrm{in}}$ & Number of samples representing an input distribution in a theoretical query or experiment. \\
$N_{\mathrm{out}}$ & Number of samples representing an output distribution in an experiment. \\
$N_{\mathrm{batch}}$ & Number of law-level examples in an optimization mini-batch. \\
$\mathcal E_\Xcal^n$, $\mathcal D_\Xcal^n$ & Finite-rank encoder and decoder for the Hilbert-space state $\Xcal$. \\
$P_\Xcal^n$ & Orthogonal projection $\mathcal D_\Xcal^n\circ\mathcal E_\Xcal^n$. \\
$\E$, $\Pr$ & Expectation and probability. \\
$\supp$, $\Lip$, $\diam$ & Support, Lipschitz constant, and diameter. \\
$\varepsilon$, $\delta$ & Approximation tolerance and confidence level. \\
\bottomrule
\end{tabular}
\end{center}

\section{Definitions and Theory Details}
\label{app:proofs}

We collect the definitions, supporting results, and proofs for the main-text
universal approximation theorems.

\subsection{Continuous finite-support construction}
\label{app:finite-construction}

\begin{lemma}[Finite separation in $W_2$]
\label{lem:finite-separation}
Let $\Kcal\subset\Pcal_2(\R^d)$ be $W_2$-compact and let $\eta>0$. Then there
are $M\in\N$, functions $\psi_1,\ldots,\psi_M\in C_b(\R^d)$, and $r>0$ such
that, for all $\mu,\nu\in\Kcal$,
\begin{equation*}
 \begin{aligned}
 &\|E_{\rm wk}(\mu)-E_{\rm wk}(\nu)\|_2<r,\\
 &|m_2(\mu)-m_2(\nu)|<r
 \quad\Longrightarrow\quad W_2(\mu,\nu)<\eta.
 \end{aligned}
\end{equation*}
\end{lemma}

\begin{proof}
Choose a countable convergence-determining family in $C_b(\R^d)$.  If no
finite subfamily worked, compactness would give sequences $\mu_k,\nu_k$ with
all first $k$ test integrals and their second moments approaching one another,
but with $W_2(\mu_k,\nu_k)\geq\eta$.  Passing to $W_2$-convergent subsequences
gives limits with identical bounded-test integrals and hence the same weak
limit.  Convergence of the second moments then characterizes $W_2$
convergence, a contradiction.
\end{proof}

\begin{theorem}[Continuous finite-support approximation]
\label{thm:finite-support}
Let $\Kcal\subset\Pcal_2(\R^{d_x})$ be $W_2$-compact and let
$\Gcal:\Kcal\to\Pcal_2(\R^{d_y})$ be $W_2$-continuous.  For every
$\varepsilon>0$, there exist $M,N\in\N$, functions
$\psi_1,\ldots,\psi_M\in C_b(\R^{d_x})$, fixed points
$a_1,\ldots,a_N\in\R^{d_y}$, and a continuous map
$p=(p_1,\ldots,p_N):\R^{M+1}\to\Delta_N$, where
$\Delta_N=\{q\in[0,1]^N:\sum_{j=1}^Nq_j=1\}$, such that
\begin{equation}
 \Ncal_0(\mu)=\sum_{j=1}^N p_j(E(\mu))\delta_{a_j}
 \label{eq:continuous-finite-support}
\end{equation}
satisfies
\begin{equation*}
 \sup_{\mu\in\Kcal}W_2(\Ncal_0(\mu),\Gcal(\mu))<\varepsilon,
\end{equation*}
with $E$ defined in \eqref{eq:test-encoder}.
\end{theorem}

\begin{proof}[Proof of Theorem~\ref{thm:finite-support}]
By uniform continuity of $\Gcal$ and Lemma~\ref{lem:finite-separation}, choose a
finite feature map $E$ and radius $r$ so that nearby features produce outputs
within $\varepsilon/4$. Cover the compact
set $E(\Kcal)$ by finitely many balls centered at $E(\mu_\ell)$.  A continuous
partition of unity $\alpha_\ell$ subordinate to this cover may be obtained
from normalized tent functions.

Finitely supported laws are dense in $\Pcal_2(\R^{d_y})$, so choose
$\widehat\nu_\ell$ with
$W_2(\widehat\nu_\ell,\Gcal(\mu_\ell))<\varepsilon/4$.  Collect the finitely
many atoms of all $\widehat\nu_\ell$ into a common list
$a_1,\ldots,a_N$ and write their weights as $\beta_{\ell j}$.  Define on
$E(\Kcal)$
\begin{equation*}
 q_j(z)=\sum_\ell\alpha_\ell(z)\beta_{\ell j},\qquad
 \Ncal_0(\mu)=\sum_jq_j(E(\mu))\delta_{a_j}.
\end{equation*}
Mixing couplings for the active $\ell$ gives
\begin{equation*}
 W_2^2(\Gcal(\mu),\Ncal_0(\mu))
 \leq\sum_\ell\alpha_\ell(E(\mu))
 W_2^2(\Gcal(\mu),\widehat\nu_\ell)<\varepsilon^2.
\end{equation*}
Finally extend the continuous map $q:E(\Kcal)\to\Delta_N$ to all of
$\R^{M+1}$ using the Dugundji extension theorem (the simplex is convex).
\end{proof}

\subsection{Neural realization and proof of Theorem~\ref{thm:neural-realization}}
\label{app:neural-proof}

\begin{lemma}[Uniform neural approximation of the features]
\label{lem:feature-approx}
Let $\mathfrak K\subset\Pcal_2(\R^d)$ be tight and satisfy
\eqref{eq:ui-second-moment}.  For finitely many $\psi_i\in C_b(\R^d)$ and
every $\rho>0$, bounded neural networks $\widehat\psi_i$ and a bounded neural
network $\widehat q_R$ exist such that
\begin{align*}
 \sup_{\mu\in\mathfrak K}\left|\int(\widehat\psi_i-\psi_i)d\mu\right|&<\rho,
 &
 \sup_{\mu\in\mathfrak K}\left|\int\widehat q_Rd\mu-m_2(\mu)\right|&<\rho.
\end{align*}
\end{lemma}

\begin{proof}
For a bounded test, choose a ball carrying uniformly high probability,
approximate the test uniformly there, and clip the network to a fixed bound;
tightness controls the complement.  For $q(x)=\|x\|_2^2$, choose $R$ so that
the tail integral in \eqref{eq:ui-second-moment} is small, multiply $q$ by a
continuous cutoff supported in the ball of radius $2R$, and uniformly
approximate and clip that compactly supported function.  Its integral differs
from $m_2$ only by the approximation error and the uniformly small tail.
\end{proof}

\begin{lemma}[Common-support stability]
\label{lem:common-support}
For $a_1,\ldots,a_N\in\R^d$, $w,v\in\Delta_N$, and
$D=\max_{j,k}\|a_j-a_k\|_2$,
\begin{equation*}
 W_2^2\!\left(\sum_jw_j\delta_{a_j},\sum_jv_j\delta_{a_j}\right)
 \leq \frac{D^2}{2}\|w-v\|_1.
\end{equation*}
\end{lemma}

\begin{proof}
Match the mass common to $w_j$ and $v_j$ at zero cost and transport the
remaining mass, whose total is $\|w-v\|_1/2$, across distances at most $D$.
\end{proof}

The same neural realization argument applies to a tight input family with
uniformly integrable second moments whenever the continuous finite-support
construction is uniform on that family.

\begin{proof}[Proof of Theorem~\ref{thm:neural-realization}]
Apply Theorem~\ref{thm:finite-support} with error $\varepsilon/2$.
Lemma~\ref{lem:feature-approx} uniformly approximates its feature vector.
The resulting feature vectors lie in a compact neighborhood of $E(\Kcal)$.
Uniform continuity of $p$ makes the induced weight error small.  To realize
$p$ by softmax, use the positive regularization
\begin{equation*}
p_j^\tau=\frac{p_j+\tau}{1+N\tau}.
\end{equation*}
Approximate the continuous logits
$\log p_j^\tau$ uniformly on that compact neighborhood, and apply softmax.
Lemma~\ref{lem:common-support} converts the final weight error to $W_2$ error.
\end{proof}

\subsection{Uniform empirical-feature bound}
\label{app:empirical-bound}

\begin{proposition}[Uniform empirical-feature error bound]
\label{prop:empirical}
Fix any neural realization from Theorem~\ref{thm:neural-realization}, and write
$f_i=\widehat\psi_i$ for $i\leq M$ and
$f_{M+1}=\widehat q_R$.  Define the realized feature ranges
\begin{align*}
 \ell_i&=\inf_{x\in\R^{d_x}}f_i(x), &
 u_i&=\sup_{x\in\R^{d_x}}f_i(x), &
 s_i&=u_i-\ell_i,
\end{align*}
the feature box $\mathcal Z=\prod_{i=1}^{M+1}[\ell_i,u_i]$, and the modulus
of continuity of the output network on this box by
\begin{equation*}
 \omega_{\widehat p,\mathcal Z}(t)
 =\sup_{\substack{z,z'\in\mathcal Z\\
                   \|z-z'\|_\infty\leq t}}
   \|\widehat p(z)-\widehat p(z')\|_1.
\end{equation*}
Let
\begin{equation*}
 \eta_\theta=\sup_{\mu\in\Kcal}
 W_2(\widehat\Ncal_\theta(\mu),\Gcal(\mu)),
 \qquad
 D=\max_{j,k}\|a_j-a_k\|_2.
\end{equation*}
If $X_1,\ldots,X_{N_{\mathrm{in}}}
\stackrel{\mathrm{i.i.d.}}\sim\mu$ and the feature
integrals in \eqref{eq:neural-realization} are replaced by empirical averages,
then, for every $t>0$,
\begin{equation}
 \sup_{\mu\in\Kcal}\Pr\!\left(
 W_2(\widehat\Ncal_{\theta,N_{\mathrm{in}}}(\mu),\Gcal(\mu))
 >\eta_\theta+D\sqrt{\frac{\omega_{\widehat p,\mathcal Z}(t)}{2}}
 \right)
 \leq 2\sum_{\substack{1\leq i\leq M+1\\s_i>0}}
 \exp\!\left(-\frac{2N_{\mathrm{in}}t^2}{s_i^2}\right).
 \label{eq:empirical-feature-tail}
\end{equation}
In particular, with $s_{\max}=\max_i s_i$ and
\begin{equation*}
 t_{N_{\mathrm{in}},\delta}=s_{\max}
 \sqrt{\frac{\log(2(M+1)/\delta)}{2N_{\mathrm{in}}}},
\end{equation*}
the following bound holds with probability at least $1-\delta$, uniformly in
$\mu\in\Kcal$:
\begin{equation}
 \sup_{\mu\in\Kcal}\Pr\!\left(
 W_2(\widehat\Ncal_{\theta,N_{\mathrm{in}}}(\mu),\Gcal(\mu))
 >\eta_\theta
 +D\sqrt{\frac{\omega_{\widehat p,\mathcal Z}
 (t_{N_{\mathrm{in}},\delta})}{2}}
 \right)\leq\delta.
 \label{eq:empirical-feature-delta}
\end{equation}
If $\widehat p$ is $L_p$-Lipschitz from $\ell_\infty$ to $\ell_1$ on
$\mathcal Z$, the second term in \eqref{eq:empirical-feature-delta} is at most
\begin{equation*}
 D\sqrt{\frac{L_p s_{\max}}{2}}
 \left(\frac{\log(2(M+1)/\delta)}{2N_{\mathrm{in}}}\right)^{1/4}.
\end{equation*}
Consequently, for every $\varepsilon>0$ and $\delta\in(0,1)$ the networks and
$N_{\mathrm{in}}$ can be chosen so that
\begin{equation*}
 \sup_{\mu\in\Kcal}\Pr\!\left(
 W_2(\widehat\Ncal_{\theta,N_{\mathrm{in}}}(\mu),\Gcal(\mu))>\varepsilon
 \right)<\delta.
\end{equation*}
\end{proposition}

\begin{proof}[Proof of Proposition~\ref{prop:empirical}]
For $\mu\in\Kcal$, set
\begin{equation*}
 z(\mu)=\left(\int f_1\,d\mu,\ldots,\int f_{M+1}\,d\mu\right),
 \qquad
 z_{N_{\mathrm{in}}}(\mu)=\frac1{N_{\mathrm{in}}}
 \sum_{b=1}^{N_{\mathrm{in}}}
 \left(f_1(X_b),\ldots,f_{M+1}(X_b)\right).
\end{equation*}
Both vectors belong to $\mathcal Z$.  For every coordinate with $s_i>0$,
Hoeffding's inequality gives
\begin{equation*}
 \Pr\!\left(|z_{N_{\mathrm{in}},i}(\mu)-z_i(\mu)|>t\right)
 \leq 2\exp\!\left(-\frac{2N_{\mathrm{in}}t^2}{s_i^2}\right).
\end{equation*}
Coordinates with $s_i=0$ have zero empirical error.  Hence a union bound,
whose right-hand side is independent of $\mu$, yields
\begin{equation}
 \sup_{\mu\in\Kcal}\Pr\!\left(
 \|z_{N_{\mathrm{in}}}(\mu)-z(\mu)\|_\infty>t\right)
 \leq 2\sum_{\substack{1\leq i\leq M+1\\s_i>0}}
 \exp\!\left(-\frac{2N_{\mathrm{in}}t^2}{s_i^2}\right).
 \label{eq:proof-feature-concentration}
\end{equation}
On the complementary event, the definition of the modulus gives
\begin{equation*}
 \|\widehat p(z_{N_{\mathrm{in}}}(\mu))-\widehat p(z(\mu))\|_1
 \leq\omega_{\widehat p,\mathcal Z}(t).
\end{equation*}
Applying Lemma~\ref{lem:common-support} to the empirical- and
population-feature outputs therefore gives
\begin{equation*}
 W_2(\widehat\Ncal_{\theta,N_{\mathrm{in}}}(\mu),
     \widehat\Ncal_\theta(\mu))
 \leq D\sqrt{\frac{\omega_{\widehat p,\mathcal Z}(t)}{2}}.
\end{equation*}
The triangle inequality and the definition of $\eta_\theta$, combined with
\eqref{eq:proof-feature-concentration}, prove
\eqref{eq:empirical-feature-tail}.

For $t=t_{N_{\mathrm{in}},\delta}$, each nonzero-range summand on the right of
\eqref{eq:empirical-feature-tail} is at most $\delta/(M+1)$ after including
its leading factor $2$, so their sum is at most $\delta$.  This proves
\eqref{eq:empirical-feature-delta}.  If $\widehat p$ is $L_p$-Lipschitz, then
$\omega_{\widehat p,\mathcal Z}(t)\leq L_pt$, which gives the displayed
Lipschitz specialization.

Finally, apply \eqref{eq:empirical-feature-delta} with confidence level
$\delta/2$.  Theorem~\ref{thm:neural-realization} permits choosing
$\eta_\theta<\varepsilon/2$.  Since $\mathcal Z$ is compact and
$\widehat p$ is continuous,
$\omega_{\widehat p,\mathcal Z}(t)\to0$ as $t\downarrow0$. Taking
$N_{\mathrm{in}}$
sufficiently large makes the second term in
\eqref{eq:empirical-feature-delta} smaller than $\varepsilon/2$, proving the
last assertion.
\end{proof}

\subsection{Hilbert-space reconstruction and proof of Theorem~\ref{thm:path-space}}
\label{app:hilbert-proof}

\begin{lemma}[Lipschitz pushforward and reconstruction]
\label{lem:pushforward}
If $F:\Xcal\to\Ycal$ is Lipschitz, then
\begin{equation*}
W_2^\Ycal(F_\#\mu,F_\#\nu)
\leq\Lip(F)W_2^\Xcal(\mu,\nu).
\end{equation*}
Moreover, for a measurable $P:\Xcal\to\Xcal$,
\begin{equation*}
 W_2^\Xcal(P_\#\mu,\mu)
 \leq\left(\int\|Px-x\|_\Xcal^2d\mu(x)\right)^{1/2}.
\end{equation*}
\end{lemma}

\begin{proof}
Push a coupling of $\mu,\nu$ through $(F,F)$ for the first inequality.  For
the second, push $\mu$ through $x\mapsto(Px,x)$.
\end{proof}

\begin{lemma}[Uniform Wasserstein autoencoding in a separable Hilbert space]
\label{lem:hilbert-autoencoding}
Let $\Xcal$ be a separable Hilbert space and let
$X_1\subset X_2\subset\cdots\subset\Xcal$ be finite-dimensional subspaces
whose union is dense.  Write $d_n=\dim X_n$, choose an orthonormal basis
$e_1^n,\ldots,e_{d_n}^n$ of $X_n$, and define
\begin{align*}
 \mathcal E_\Xcal^n x
   &=\bigl(\langle x,e_1^n\rangle_\Xcal,\ldots,
       \langle x,e_{d_n}^n\rangle_\Xcal\bigr),\\
 \mathcal D_\Xcal^n z
   &=\sum_{j=1}^{d_n}z_j e_j^n,
 \qquad
 P_\Xcal^n=\mathcal D_\Xcal^n\circ\mathcal E_\Xcal^n.
\end{align*}
Then $\mathcal E_\Xcal^n$ is a contraction,
$\mathcal D_\Xcal^n$ is an isometry, and $P_\Xcal^n$ is the orthogonal
projection onto $X_n$.  Moreover, every $W_2^\Xcal$-compact set
$\Kcal\subset\Pcal_2(\Xcal)$ satisfies
\begin{equation}
 \lim_{n\to\infty}\sup_{\mu\in\Kcal}
 W_2^\Xcal\bigl((P_\Xcal^n)_\#\mu,\mu\bigr)=0.
 \label{eq:hilbert-wass-autoencoder}
\end{equation}
\end{lemma}

\begin{proof}
The norm claims and the identity of $P_\Xcal^n$ follow from orthonormality.
For $\mu\in\Pcal_2(\Xcal)$, set
\begin{equation*}
 r_n(\mu)=\left(\int_\Xcal
       \|(I-P_\Xcal^n)x\|_\Xcal^2\,d\mu(x)\right)^{1/2}.
\end{equation*}
The coupling $x\mapsto(P_\Xcal^n x,x)$ and
Lemma~\ref{lem:pushforward} give
$W_2^\Xcal((P_\Xcal^n)_\#\mu,\mu)\leq r_n(\mu)$.
Because $P_\Xcal^n x\to x$ and
$\|(I-P_\Xcal^n)x\|_\Xcal\leq\|x\|_\Xcal$, dominated convergence gives
$r_n(\mu)\to0$ for each $\mu$.

The convergence is uniform on $\Kcal$.  Indeed, for any coupling $\pi$ of
$\mu$ and $\nu$, the reverse triangle inequality in $L^2(\pi;\Xcal)$ and the
contractivity of $I-P_\Xcal^n$ yield
\begin{equation*}
 |r_n(\mu)-r_n(\nu)|
 \leq\left(\int\|(I-P_\Xcal^n)(x-y)\|_\Xcal^2
       \,d\pi(x,y)\right)^{1/2}
 \leq\left(\int\|x-y\|_\Xcal^2\,d\pi(x,y)\right)^{1/2}.
\end{equation*}
Taking the infimum over $\pi$ shows that every $r_n$ is $1$-Lipschitz on
$\Pcal_2(\Xcal)$.  Pointwise convergence together with this common Lipschitz
bound and compactness of $\Kcal$ implies
$\sup_{\mu\in\Kcal}r_n(\mu)\to0$, proving
\eqref{eq:hilbert-wass-autoencoder}.
\end{proof}

\begin{proof}[Proof of Theorem~\ref{thm:path-space}]
Write $P_\Xcal^n=\mathcal D_\Xcal^n\circ\mathcal E_\Xcal^n$ and define
$P_\Ycal^m$ similarly.  Lemma~\ref{lem:hilbert-autoencoding} gives
\begin{equation}
 \sup_{\mu\in\Kcal}
 W_2^\Xcal\bigl((P_\Xcal^n)_\#\mu,\mu\bigr)\longrightarrow0.
 \label{eq:input-projection-uniform}
\end{equation}
Since $\Kcal$ is compact and $\Ucal$ is an open neighborhood of $\Kcal$,
\eqref{eq:input-projection-uniform} implies that
$(P_\Xcal^n)_\#\mu\in\Ucal$ for every $\mu\in\Kcal$ and all sufficiently
large $n$.  Continuity of $\Gcal$ then implies
\begin{equation}
 \sup_{\mu\in\Kcal}W_2^\Ycal\bigl(
 \Gcal((P_\Xcal^n)_\#\mu),\Gcal(\mu)\bigr)\longrightarrow0.
 \label{eq:target-input-projection}
\end{equation}
To see the uniformity in \eqref{eq:target-input-projection}, otherwise choose
$n_k\to\infty$ and $\mu_k\in\Kcal$ violating it.  After passing to a
subsequence, $\mu_k\to\mu\in\Kcal$ in $W_2^\Xcal$; then
\eqref{eq:input-projection-uniform} also gives
$(P_\Xcal^{n_k})_\#\mu_k\to\mu$, contradicting continuity of $\Gcal$ at
$\mu$.  Fix $n$ so that the left-hand side of
\eqref{eq:target-input-projection} is less than $\varepsilon/3$.

The family
\begin{equation*}
 \mathfrak L_n=
 \{\Gcal((P_\Xcal^n)_\#\mu):\mu\in\Kcal\}
 \subset\Pcal_2(\Ycal)
\end{equation*}
is compact.  Applying Lemma~\ref{lem:hilbert-autoencoding} in $\Ycal$, choose
$m$ such that
\begin{equation}
 \sup_{\nu\in\mathfrak L_n}
 W_2^\Ycal((P_\Ycal^m)_\#\nu,\nu)<\varepsilon/3.
 \label{eq:output-projection-uniform}
\end{equation}
The encoded input family
$\Kcal_n=(\mathcal E_\Xcal^n)_\#\Kcal$ is compact.  On $\Kcal_n$ define
\begin{equation*}
 \mathcal S_{n,m}(\lambda)=(\mathcal E_\Ycal^m)_\#
 \Gcal((\mathcal D_\Xcal^n)_\#\lambda).
\end{equation*}
Both pushforwards are Lipschitz by Lemma~\ref{lem:pushforward}, so
$\mathcal S_{n,m}$ is a continuous operator between finite-dimensional law
spaces.  Theorem~\ref{thm:neural-realization} therefore provides
$\widehat\Ncal_\theta$ such that
\begin{equation}
 \sup_{\lambda\in\Kcal_n}W_2\bigl(
 \widehat\Ncal_\theta(\lambda),\mathcal S_{n,m}(\lambda)\bigr)
 <\varepsilon/3.
 \label{eq:encoded-neural-error}
\end{equation}
For $\mu\in\Kcal$, put
$\lambda=(\mathcal E_\Xcal^n)_\#\mu$ and
$\nu=\Gcal((P_\Xcal^n)_\#\mu)$.  Since
$\mathcal D_\Ycal^m$ is an isometry,
Lemma~\ref{lem:pushforward}, \eqref{eq:encoded-neural-error},
\eqref{eq:output-projection-uniform}, and
\eqref{eq:target-input-projection} give
\begin{align*}
 W_2^\Ycal(\Ncal_{\theta,n,m}(\mu),\Gcal(\mu))
 &\leq W_2^\Ycal\bigl(
   (\mathcal D_\Ycal^m)_\#\widehat\Ncal_\theta(\lambda),
   (P_\Ycal^m)_\#\nu\bigr)\\
 &\quad+W_2^\Ycal((P_\Ycal^m)_\#\nu,\nu)
   +W_2^\Ycal(\nu,\Gcal(\mu))
 <\varepsilon.
\end{align*}
Taking the supremum over $\mu\in\Kcal$ proves the claim.
\end{proof}

\subsection{Additional path-law construction details}
\label{app:path-law-details}

\input{path_law_theory.tex}

\section{Supplementary Theory Proof For Experiments}
\label{app:structured-output-theory}

\subsection{Structured Output Representations and Neural Realizations}

General finite-support output measures are expressive, but their support size
can grow rapidly with the covering number of the input set and the complexity
of local output laws. Practical problems often provide additional output
structure. We consider common compact support, finite-dimensional parametric
families, generative pushforwards, and conditional densities.

\begin{lemma}[Finite-dimensional factorization]
\label{lem:structured-factorization}
Let $\Kcal\subset\Pcal_2(\R^{d_x})$ be $W_2$-compact, let $\Ycal$ be a Banach
space, and let $F:\Kcal\to\Ycal$ be continuous. For every $\varepsilon>0$,
there are a finite-dimensional encoder $E:\Kcal\to\R^M$ and a continuous map
$H:\R^M\to\Ycal$ such that
\begin{equation*}
\sup_{\mu\in\Kcal}\|H(E(\mu))-F(\mu)\|_{\Ycal}<\varepsilon.
\end{equation*}
If $F(\Kcal)$ lies in a convex set $\Ycal_0$, then $H(E(\Kcal))$ may be
required to lie in $\Ycal_0$.
\end{lemma}

\begin{proof}
Use Lemma~\ref{lem:finite-separation} and a finite cover of $E(\Kcal)$ with a partition of unity
$\alpha_\ell$. For centers $\mu_\ell$, define
\begin{equation*}
H(z)=\sum_{\ell=1}^{L}\alpha_\ell(z)F(\mu_\ell).
\end{equation*}
Uniform continuity controls the error. This is a convex combination, so it
preserves $\Ycal_0$. Extend $H$ from the compact feature set by the
Tietze--Dugundji theorem.
\end{proof}

\subsubsection{Common Support on a Compact Set}

\begin{proposition}[Shared support points]
\label{prop:structured-shared-support}
Suppose the assumptions of Theorem~\ref{thm:finite-support} hold and there is a compact
$D\subset\R^{d_y}$ such that
$\supp\Gcal(\mu)\subset D$ for every $\mu\in\Kcal$. For every
$\varepsilon>0$, there are common points $a_1,\ldots,a_N\in D$, an encoder
$E$, and continuous simplex-valued weights $p_j$ such that
\begin{equation*}
\Ncal(\mu)=\sum_{j=1}^{N}p_j(E(\mu))\delta_{a_j},\qquad
\sup_{\mu\in\Kcal}W_2(\Ncal(\mu),\Gcal(\mu))<\varepsilon.
\end{equation*}
The support size $N$ may be bounded by the covering number of $D$ at a scale
proportional to $\varepsilon$.
\end{proposition}

\begin{proof}
Choose an $r$-net $\{a_j\}_{j=1}^{N}$ of $D$ and a continuous partition of
unity $\{\lambda_j\}$ subordinate to it. For any law $\nu$ on $D$, set
\begin{equation*}
Q(\nu)=\sum_{j=1}^{N}\left(\int_D\lambda_j(y)d\nu(y)\right)\delta_{a_j}.
\end{equation*}
The coupling $\lambda_j(y)d\nu(y)$ gives $W_2(\nu,Q(\nu))\leq r$. The map
$q(\mu)=(\int\lambda_1d\Gcal(\mu),\ldots,\int\lambda_Nd\Gcal(\mu))$ is
continuous. Apply Lemma~\ref{lem:structured-factorization} to factor $q$ through $E$ and choose the weight
error so that the common-support coupling bound is below $\varepsilon-r$.
\end{proof}

\subsubsection{Finite-Dimensional Parametric Output Families}

Let $\mathfrak F=\{\nu_a:a\in\Theta\}\subset\Pcal_2(\R^{d_y})$ be a
parametric family. Suppose $\Gcal(\mu)=\nu_{A(\mu)}$ for a continuous parameter
map $A:\Kcal\to\Theta$.

\begin{proposition}[Parametric output family]
\label{prop:structured-parametric}
Assume $a\mapsto\nu_a$ is $W_2$-continuous on a compact convex set
$K_\Theta\subset\Theta$ containing $A(\Kcal)$. For every $\varepsilon>0$,
there are $E$ and continuous $H:E(\Kcal)\to K_\Theta$ such that
\begin{equation*}
\sup_{\mu\in\Kcal}W_2(\nu_{H(E(\mu))},\Gcal(\mu))<\varepsilon.
\end{equation*}
\end{proposition}

\begin{proof}
Uniform continuity of $a\mapsto\nu_a$ converts a sufficiently small parameter
error into a $W_2$ error below $\varepsilon$. Apply Lemma~\ref{lem:structured-factorization} to $A$ while
preserving the convex parameter set.
\end{proof}

Continuity of $A$ is an additional assumption: continuity of $\Gcal$ does not
by itself imply a continuous parameter selection. Nonidentifiability, label
switching in mixture models, or a discontinuous inverse parameterization can
prevent such a selection. Legal parameters must also be enforced. For Gaussian
outputs, a Cholesky factor or a positive transform of eigenvalues guarantees a
positive semidefinite covariance.

\begin{corollary}[Neural realization of a parametric output family]
\label{cor:structured-parametric-neural}
Under Proposition~\ref{prop:structured-parametric}, suppose there is $r>0$ such that
\begin{equation*}
C_r=\{a\in\R^p:\dist(a,A(\Kcal))\leq r\}\subset K_\Theta.
\end{equation*}
For every $\varepsilon>0$, there are input-feature networks and a
finite-dimensional network $\widehat H_\theta:\R^M\to\R^p$ such that
\begin{equation*}
\widehat\Ncal_\theta(\mu)=\nu_{\widehat H_\theta(\widehat E(\mu))},\qquad
\sup_{\mu\in\Kcal}W_2(\widehat\Ncal_\theta(\mu),\Gcal(\mu))<\varepsilon.
\end{equation*}
\end{corollary}

\subsubsection{Generative Pushforward Representation}

Let $\rho\in\Pcal_2(\R^{d_z})$ be a fixed latent law. Suppose there is a
measurable map $G:\Kcal\times\R^{d_z}\to\R^{d_y}$ with
$\Gcal(\mu)=\pushforward{G(\mu,\cdot)}{\rho}$.

\begin{proposition}[Generative pushforward]
\label{prop:structured-generator}
If $\mu\mapsto G(\mu,\cdot)$ is continuous from $(\Kcal,W_2)$ to
$L^2(\rho;\R^{d_y})$, then for every $\varepsilon>0$ there are an encoder $E$
and continuous $H:E(\Kcal)\to L^2(\rho;\R^{d_y})$ such that
\begin{equation*}
\Ncal_H(\mu)=\pushforward{H(E(\mu))}{\rho},\qquad
\sup_{\mu\in\Kcal}W_2(\Ncal_H(\mu),\Gcal(\mu))<\varepsilon.
\end{equation*}
\end{proposition}

\begin{proof}
Apply Lemma~\ref{lem:structured-factorization} in $L^2(\rho;\R^{d_y})$ with error $\varepsilon$. Coupling the
two outputs by the same $Z\sim\rho$ gives
\begin{equation*}
W_2^2(\Ncal_H(\mu),\Gcal(\mu))
\leq\|H(E(\mu))-G(\mu,\cdot)\|_{L^2(\rho)}^2.
\end{equation*}
\end{proof}

The existence of a generator for each individual law does not imply a
continuous choice of generators across a family. The pushforward map is many
to one, and mode permutations or phase choices can make a generator selection
discontinuous even when the laws vary continuously.

Polynomial chaos is one example. If $Z\sim\rho$ and
$Y=f(Z)=\sum_{i=0}^{\infty}a_i\psi_i(Z)$ in $L^2(\rho)$, then the truncation
$f_M(Z)=\sum_{i=0}^{M}a_i\psi_i(Z)$ satisfies
\begin{equation*}
W_2^2(\pushforward{f}{\rho},\pushforward{f_M}{\rho})
\leq\|f-f_M\|_{L^2(\rho)}^2.
\end{equation*}
For standard Gaussian $\rho$, the basis may be the Hermite polynomials. If the
coefficients depend continuously on the input law, a network can predict them
from finite distributional features.

\begin{corollary}[Neural-operator realization of a generator]
\label{cor:structured-generator-neural}
Under Proposition~\ref{prop:structured-generator}, for every $\varepsilon>0$ there are input-feature
networks and a deterministic neural operator
$\widehat H_\theta:\R^M\to L^2(\rho;\R^{d_y})$ such that
\begin{equation*}
\widehat\Ncal_\theta(\mu)=
\pushforward{\widehat H_\theta(\widehat E(\mu))}{\rho},
\qquad
\sup_{\mu\in\Kcal}W_2(\widehat\Ncal_\theta(\mu),\Gcal(\mu))<\varepsilon.
\end{equation*}
\end{corollary}

\subsubsection{Conditional Density Representation}

\begin{proposition}[Existence of a conditional-density representation]
\label{prop:structured-density}
Let $D\subset\R^{d_y}$ be compact with positive Lebesgue measure. Suppose every
$\Gcal(\mu)$ is supported on $D$ and has density $p_\mu$, and the map
$P:\Kcal\to L^1(D)$, $P(\mu)=p_\mu$, is continuous. For every
$\varepsilon>0$, there are $E$ and a continuous map
$H:E(\Kcal)\to L^1(D)$ such that every $H(z)$ is a probability density and
\begin{equation*}
\Ncal_H(\mu)(dy)=H(E(\mu))(y)dy,\qquad
\sup_{\mu\in\Kcal}W_2(\Ncal_H(\mu),\Gcal(\mu))<\varepsilon.
\end{equation*}
\end{proposition}

\begin{proof}
The density set
\begin{equation*}
\mathcal A(D)=\left\{p\in L^1(D):p\geq0\text{ a.e.},\ \int_Dp(y)dy=1\right\}
\end{equation*}
is closed and convex. Apply Lemma~\ref{lem:structured-factorization} to $P$ in $L^1(D)$ while preserving
$\mathcal A(D)$. For laws supported on $D$,
\begin{equation*}
W_2^2(\Ncal_H(\mu),\Gcal(\mu))
\leq\frac{\diam(D)^2}{2}\|H(E(\mu))-p_\mu\|_{L^1(D)}.
\end{equation*}
\end{proof}

A normalized exponential preserves positivity and unit mass:
\begin{equation*}
\mathfrak S(h)(y)=\frac{\exp(h(y))}{\int_D\exp(h(\zeta))d\zeta}.
\end{equation*}
Alternatively, for a measurable partition $\{Q_j\}_{j=1}^{J}$ and simplex
weights $\alpha_j$, the bin density
\begin{equation*}
p_\alpha(y)=\sum_{j=1}^{J}\frac{\alpha_j}{|Q_j|}\mathbf1_{Q_j}(y)
\end{equation*}
is automatically valid and converges in $L^1$ as the partition is refined.

\begin{lemma}[Stability of the normalized exponential]
\label{lem:normalized-exponential-stability}
For compact $D$ of positive Lebesgue measure and $h_1,h_2\in C(D)$,
\begin{equation*}
\|\mathfrak S(h_1)-\mathfrak S(h_2)\|_{L^1(D)}
\leq2\|h_1-h_2\|_{C(D)}.
\end{equation*}
\end{lemma}

\begin{proof}
Set $r=h_1-h_2$, $h_t=h_2+tr$, and $p_t=\mathfrak S(h_t)$. Then
\begin{equation*}
\partial_tp_t(y)=p_t(y)\left(r(y)-\int_Dr(\zeta)p_t(\zeta)d\zeta\right),
\end{equation*}
so $\|\partial_tp_t\|_{L^1(D)}\leq2\|r\|_{C(D)}$. Integrating over
$t\in[0,1]$ proves the result.
\end{proof}

\begin{corollary}[Neural realization of normalized-exponential densities]
\label{cor:structured-density-neural}
Under Proposition~\ref{prop:structured-density}, suppose a continuous energy map
$A:\Kcal\to C(D)$ satisfies
\begin{equation*}
p_\mu(y)=\frac{\exp(A(\mu)(y))}{\int_D\exp(A(\mu)(\zeta))d\zeta}.
\end{equation*}
For every $\varepsilon>0$, there are feature networks and a neural operator
$\widehat H_\theta:\R^M\to C(D)$ such that
\begin{align*}
\widehat p_\theta(y\mid\mu)&=
\frac{\exp(\widehat H_\theta(\widehat E(\mu))(y))}
{\int_D\exp(\widehat H_\theta(\widehat E(\mu))(\zeta))d\zeta},\\
\widehat\Ncal_\theta(\mu)(dy)&=\widehat p_\theta(y\mid\mu)dy,
\end{align*}
and
\begin{equation*}
\sup_{\mu\in\Kcal}W_2(\widehat\Ncal_\theta(\mu),\Gcal(\mu))<\varepsilon.
\end{equation*}
\end{corollary}

\begin{proof}
Approximate $A$ after finite-dimensional factorization and input-feature
approximation, obtaining
\begin{equation*}
\sup_{\mu\in\Kcal}
\|\widehat H_\theta(\widehat E(\mu))-A(\mu)\|_{C(D)}<\gamma.
\end{equation*}
Lemma~\ref{lem:normalized-exponential-stability} gives an $L^1$ density error below $2\gamma$, and therefore
\begin{equation*}
W_2^2(\widehat\Ncal_\theta(\mu),\Gcal(\mu))
\leq\diam(D)^2\gamma.
\end{equation*}
Choose $\gamma<\varepsilon^2/\diam(D)^2$.
\end{proof}

If $p_\mu$ is strictly positive and $\mu\mapsto\log p_\mu$ is continuous into
$C(D)$, one may take $A(\mu)=\log p_\mu$. Densities with zeros require a
positive regularization or bin-density approximation first.

\section{Data Generation and Experiments Protocol}
\label{app:experiment-protocols}

\subsection{OU First-Passage Time}
\label{app:ou-details}

\subsubsection{Data Generation}
\label{app:ou-data}

The 1,200 laws are divided equally among three drift regimes:
\begin{equation*}
\begin{array}{lll}
\text{subthreshold}:&m\sim\Unif(0.75,0.95),&400\text{ laws},\\
\text{balanced}:&m\sim\Unif(0.95,1.05),&400\text{ laws},\\
\text{suprathreshold}:&m\sim\Unif(1.05,1.25),&400\text{ laws}.
\end{array}
\end{equation*}
The master seed is zero. Parameter sampling uses seed 101, and $q$ is sampled
independently with
\begin{equation*}
\log q\sim\Unif(\log 0.1,\log 0.35).
\end{equation*}
Each law is assigned a unique identifier from 0 through 1199. For the target,
200 independent reset trials follow
\begin{align*}
dV_t&=(-V_t+m)dt+\sqrt q\,dW_t,\\
V_0&=0,
\end{align*}
using random seed 303, step size $0.002$, threshold 1, and terminal time 8. A
crossing within a step is linearly interpolated; trials without a crossing by
time 8 are censored. Uncensored times are assigned to 48 equal-width bins
with edges $t_b=8b/48$, and censored trials form category 49. Dividing the 49
category counts by 200 gives the target mass.

The input process file is generated from the same parameter table. For every
law, 200 independent paths satisfy
\begin{equation*}
dX_t=m\,dt+\sqrt q\,dW_t,\qquad X_0=0,
\end{equation*}
and are evaluated at $t_k=8k/255$, $k=0,\ldots,255$. Independent Gaussian
increments with variance $8/255$ give the exact transition law on this grid,
and cumulative summation preserves the temporal correlation within each
path. A deterministic stream generated by
\texttt{SeedSequence([0, law\_id])} makes every law reproducible independently
of generation order. The resulting array has shape
$1200\times200\times256$ and is stored in single precision. The input and
target files are joined by their unique law identifiers rather than by row
position, with matched parameters and regime labels verified before use.

A deterministic regime-stratified split with seed zero assigns 1,000 complete
laws to training and 200 to testing; all paths from a law remain in the same
split. There is no validation split, and the test laws are excluded from PCA
fitting, optimization, scheduling, early stopping, and checkpoint selection.

\begin{figure}[htbp]
\centering
\includegraphics[width=0.80\textwidth]{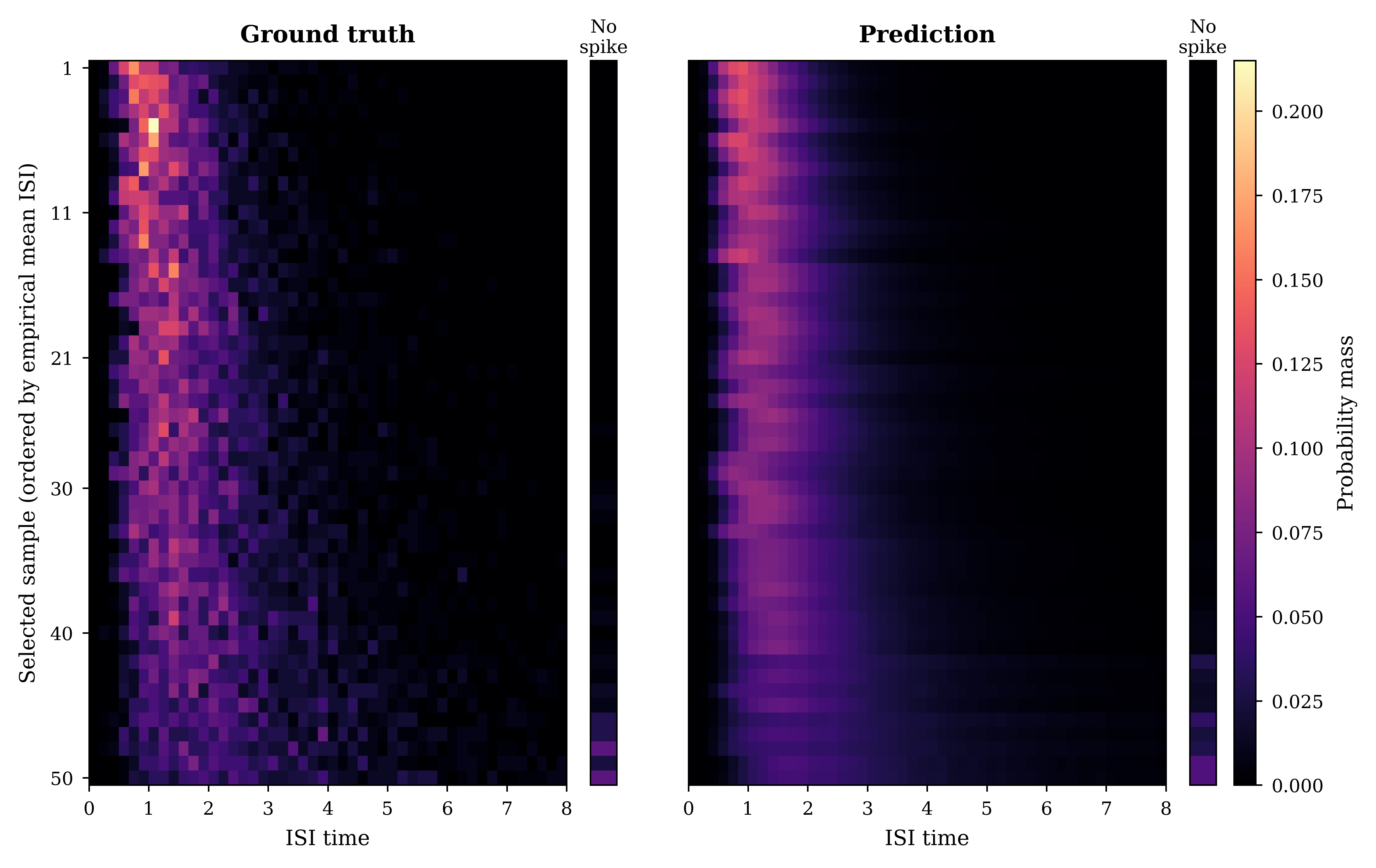}
\caption{Predicted and empirical interspike-interval distributions.
Ground-truth probability masses (left) and those predicted by the Process
Distributional Operator (right) are shown for 50 test examples selected
uniformly at random and ordered by increasing empirical mean interspike
interval (ISI). Each heatmap row represents one input law, and the columns
span the 48 finite ISI bins from 0 to 8. The narrow column beside each heatmap
shows the probability of no threshold crossing before the end of the
observation window. All panels use the same color scale.}
\label{fig:ou-test-heatmap}
\end{figure}

\subsubsection{Model and Baseline Implementations}
\label{app:ou-baselines}
For the Distributional Operator, one global PCA is fitted to the
200,000 training paths after flattening the law and path axes. The 256 time
coordinates are centered using their training means but are not standardized
coordinatewise. The mean and covariance are accumulated in double precision,
and a symmetric eigendecomposition orders the components by decreasing
eigenvalue. The smallest dimension whose cumulative explained-variance ratio
is strictly greater than $0.99$ is 14, with achieved ratio $0.99015011$.
The fitted mean and components are then applied unchanged to both splits. The
PCA state, threshold, split seed, and training-law identifiers are stored with
the checkpoint so evaluation never refits the transform.

The resulting model input has shape
$N_{\mathrm{batch}}\times N_{\mathrm{in}}\times14$. Each compressed
path passes through a shared $14\to32\to32$ SiLU path MLP. A DeepSets element
network then maps $32\to32\to32$ with a SiLU hidden activation and no dropout,
and averages the resulting 32-dimensional embeddings over the input-sample axis.
Finally, a $32\to32\to49$ outer network with a GELU hidden activation returns
the logits. The complete sequence of tensor shapes is
\begin{equation*}
\begin{aligned}
(N_{\mathrm{batch}},N_{\mathrm{in}},256)
&\xrightarrow{\mathrm{PCA}}(N_{\mathrm{batch}},N_{\mathrm{in}},14)\\
&\xrightarrow{\phi_{\mathrm{path}}}(N_{\mathrm{batch}},N_{\mathrm{in}},32)\\
&\xrightarrow{\mathrm{DeepSets\ mean}}(N_{\mathrm{batch}},32)
\xrightarrow{\rho}(N_{\mathrm{batch}},49).
\end{aligned}
\end{equation*}
The neural model has 5,265 trainable parameters. Softmax returns the predicted
bin masses and cumulative summation returns the predicted CDF. It is trained
for 1,000 epochs with AdamW, learning rate $10^{-3}$, weight decay $10^{-2}$,
mini-batches of 64 laws, no learning-rate scheduler, and no early stopping;
the saved final checkpoint is evaluated once on the untouched test split.

\begin{figure}[H]
\centering
\includegraphics[width=\textwidth]{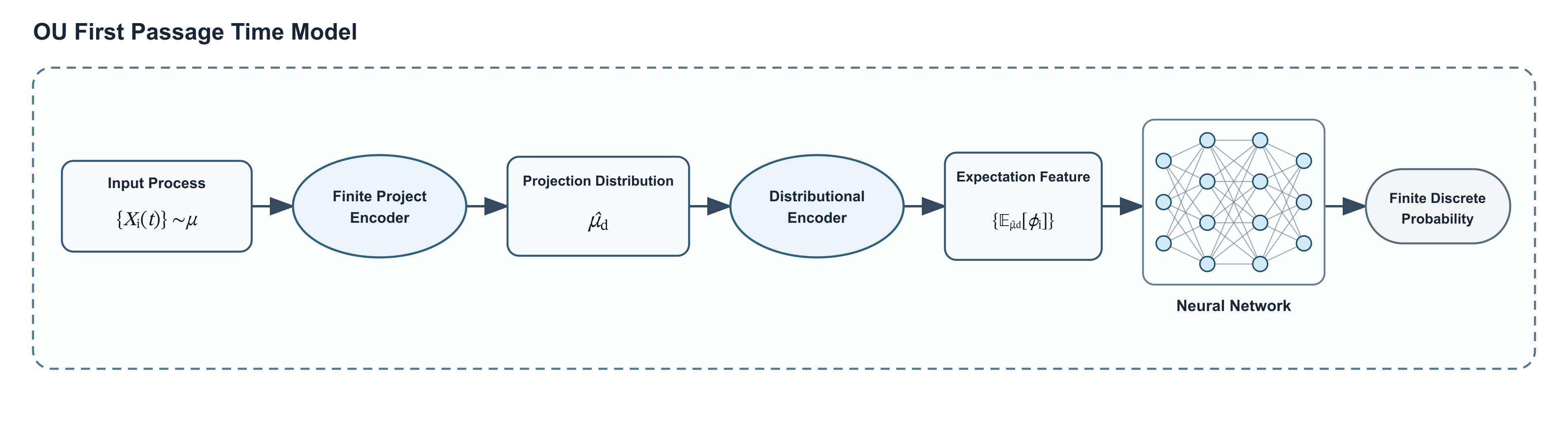}
\caption{Distributional Operator used for the OU first-passage task.
An empirical law of PCA-compressed input paths is encoded by shared path and
DeepSets networks, mapped to categorical output-law parameters, and sampled
when realizations of the predicted first-passage distribution are required.}
\label{fig:ou-distributional-operator-pipeline}
\end{figure}

The kernel baseline converts the input particles into a normalized
32-category histogram over the global training-particle range. For query $X$
and training law $X_i$, let $d(X,X_i)$ be the exact one-dimensional quadratic
Wasserstein distance between the piecewise-uniform histograms. It predicts
\begin{equation*}
\widehat p(X)=\sum_iw_i(X)p_i,\qquad
w_i(X)=\frac{\exp[-d(X,X_i)^2/(2h^2)]}
{\sum_j\exp[-d(X,X_j)^2/(2h^2)]},\qquad h=0.15.
\end{equation*}
Self-matches are excluded for training-law evaluation, and distances are
computed in blocks of 256 training laws.

The fixed-feature MLP replaces learned particle aggregation by the sample mean,
population variance, and eight empirical sine and cosine moments at frequencies
drawn once from $\mathcal N(0,1)$. The resulting 18 features pass through a
linear $18\to32$ map, GELU, and a $32\to49$ logit map. It is trained for
1,000 epochs with mini-batches of 64 laws using the count-weighted categorical
NLL objective.

\subsubsection{Evaluation Metrics}
\label{app:ou-metrics}
Five runs use mini-batches of 64 under the default training configuration.

For normalized target histogram $p_i$ and prediction $\widehat p_i$, we report
count-normalized NLL, Hellinger distance, empirical-to-prediction KL divergence,
tail error, and finite-support $W_2$:
\begin{align*}
\mathrm{NLL}_i&=-\sum_{k=1}^{49}p_{ik}\log\widehat p_{ik},\\
H(p_i,\widehat p_i)&=\left[\frac12\sum_{k=1}^{49}
(\sqrt{p_{ik}}-\sqrt{\widehat p_{ik}})^2\right]^{1/2},\\
D_{\mathrm{KL}}(p_i\Vert\widehat p_i)&=
\sum_{k:p_{ik}>0}p_{ik}\log\frac{p_{ik}}{\widehat p_{ik}},\\
E_{\mathrm{tail}}(p_i,\widehat p_i)&=
|p_{i,49}-\widehat p_{i,49}|.
\end{align*}
For finite-support $W_2$, category 49 is removed, the first 48 masses are
renormalized, and the exact one-dimensional quadratic Wasserstein distance is
computed between the associated piecewise-uniform laws. Metrics are averaged
over the 200 test laws and then summarized across the five runs.

\subsection{Nonlinear Duffing Oscillator}
\label{app:duffing-data}

\subsubsection{Data Generation}

The Duffing coefficients are fixed throughout the dataset at linear stiffness
$a=1$, cubic stiffness $b=1$, and damping $c=0.2$. A forcing-law parameter is
\begin{equation*}
\theta=(A,f,\sigma_\eta,\ell_\eta)\in
[0.5,2.0]\times[0.10,0.30]\times[0.05,0.40]\times[0.10,1.00].
\end{equation*}
For each realization, independently sample $\Phi\sim\Unif[0,2\pi]$ and a
zero-mean Gaussian process $\eta$ with squared-exponential covariance
\begin{equation*}
K_\theta(s,t)=\sigma_\eta^2
\exp\!\left(-\frac{(s-t)^2}{2\ell_\eta^2}\right).
\end{equation*}
The resulting stochastic forcing path is
\begin{equation*}
X(t)=A\sin(2\pi ft+\Phi)+\eta(t).
\end{equation*}
Thus $(A,f,\sigma_\eta,\ell_\eta)$ identifies a probability law, whereas the
phase and Gaussian-process path are realization-level random variables and
are resampled for every particle.

We use a Latin-hypercube design to choose 1,200 parameter vectors over the
four-dimensional box above, providing space-filling coverage near potentially
sharp response transitions. The split is performed at the measure level:
1,000 settings are used for training and 200 for testing.
No realizations from one forcing law appear in multiple splits. An optional
extrapolation benchmark may instead reserve boundary bands in $A$ or $f$, but
it is kept separate from the primary interpolation split.

Both forcing and displacement paths are stored at $N_t=256$ equally spaced
observation times
\begin{equation*}
t_k=\frac{20k}{255},\qquad k=0,\ldots,255.
\end{equation*}
The ODE is not solved directly on this coarse grid. Each forcing realization
is generated on a fine grid with step size
$\Delta t_{\mathrm{solve}}=0.005$ and interpolated inside the first-order
system
\begin{equation*}
\dot y_1=y_2,\qquad
\dot y_2=X(t)-0.2y_2-y_1-y_1^3.
\end{equation*}
We integrate from zero initial conditions with adaptive Dormand--Prince RK45,
using relative tolerance $10^{-8}$ and absolute tolerance $10^{-10}$, and
then sample the solution on the observation grid. Before full generation, a
convergence study on representative parameter settings compares
$\Delta t\in\{0.02,0.01,0.005,0.0025\}$ and requires relative trajectory
error below $10^{-4}$ against the finest reference.

For each parameter vector $\theta_i$, let $\mu_i$ be the corresponding
forcing-process law. Draw two independent ensembles,
\begin{equation*}
X_{i,j}\overset{\mathrm{iid}}{\sim}\mu_i,\qquad
X'_{i,j}\overset{\mathrm{iid}}{\sim}\mu_i,\qquad j=1,\ldots,128,
\end{equation*}
and propagate only the second ensemble through the fixed Duffing solution
operator, $Y_{i,j}=\mathcal S(X'_{i,j})$. The empirical input and target laws
are therefore
\begin{align*}
\widehat\mu_X^{(i)}&=\frac1{128}\sum_{j=1}^{128}\delta_{X_{i,j}},\\
\widehat\mu_Y^{(i)}&=\frac1{128}\sum_{j=1}^{128}\delta_{Y_{i,j}}.
\end{align*}
This deliberately removes trajectory-level correspondence while retaining the
shared measure parameter. The stored arrays $X$ and $Y$ each have shape
$1,200\times128\times256$ in \texttt{measure,sample,time} order, and a
$1,200\times4$ metadata array stores the parameter vectors for diagnostics
and split construction, not as model input.

For the neural configurations, a fixed random permutation assigns the first
1,000 law indices to training and the remaining 200 to testing. The training
loader is shuffled; the test loader is not shuffled. Mini-batches contain eight laws, use no worker
processes, and contain all 128 particles from each selected law.

\subsubsection{Distributional Operator Implementation}
\label{app:duffing-model}

\begin{figure}[H]
\centering
\includegraphics[width=\textwidth]{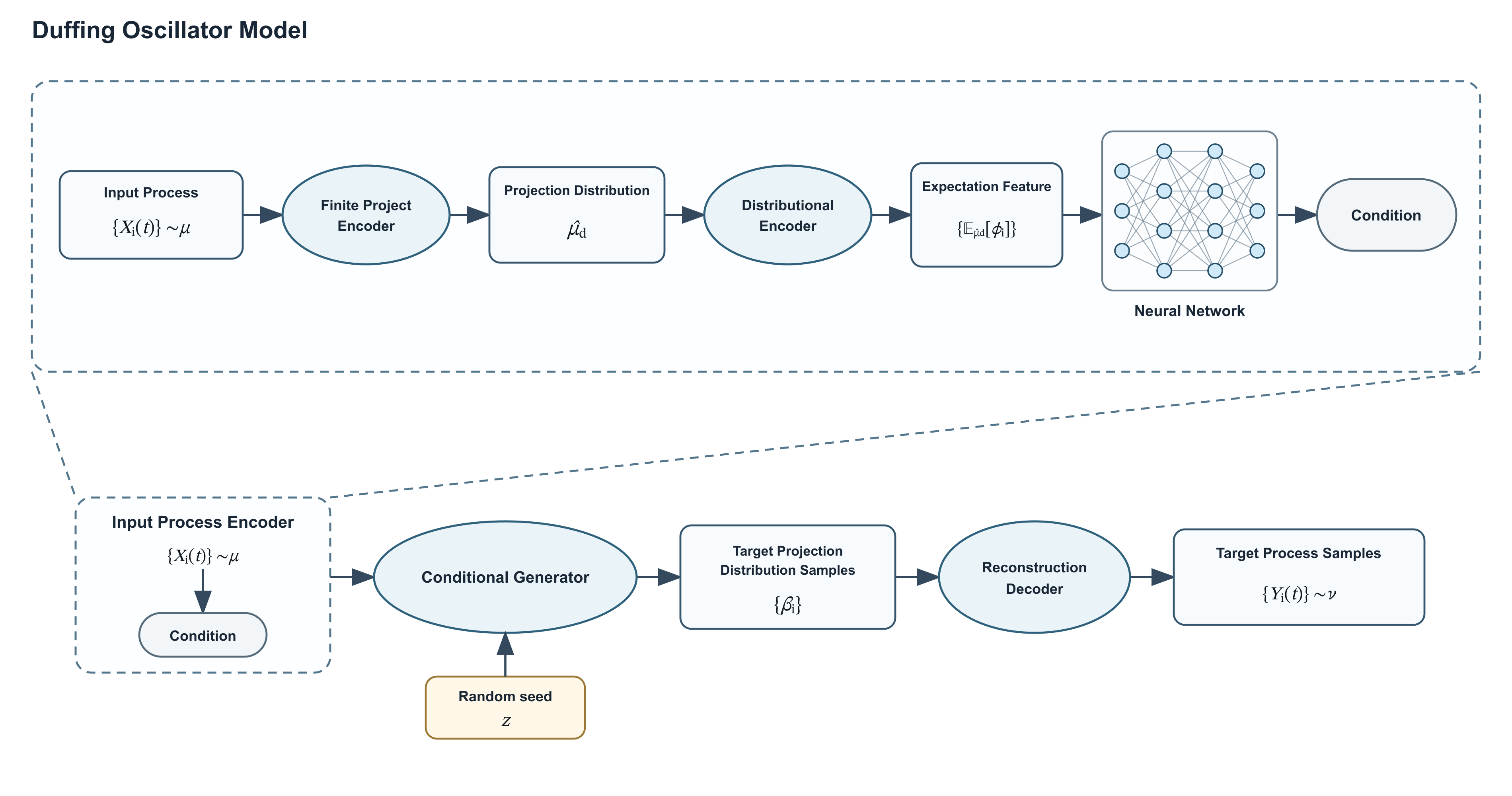}
\caption{Process-valued Distributional Operator used for the Duffing task.
Top: the input-process ensemble is projected to finite-dimensional coordinate
laws and encoded into a condition through expectation features. Bottom: the
condition and independent latent draws drive a conditional generator, whose
projected output samples are reconstructed as response-process paths.}
\label{fig:duffing-process-operator-pipeline}
\end{figure}

Separate PCA transforms are fitted to the pooled training forcing and response
trajectories, never to test data. The 256-dimensional forcing
and response grids are truncated to 24 and 13 principal coordinates,
respectively. PCA fitting accumulates the training mean and scatter matrix in
64-law blocks in double precision and uses a symmetric eigendecomposition.
Each retained coordinate is then standardized using its population mean and
standard deviation over all trajectories in the training laws; test data
reuse these stored transformations. Near-constant coordinates use
unit scale. Inputs to the learned encoder are the normalized 24-dimensional
forcing coefficients. Generated normalized 13-dimensional response
coefficients are unstandardized and reconstructed on the original 256-point
grid before the training loss and reported metrics are evaluated.

The DeepSets particle map and outer map each have two width-256 SiLU hidden
layers. The particle embedding and law context both have dimension 128, and
the aggregation is the sample mean. The conditional RealNVP has six affine
coupling layers with alternating even--odd masks. Each scale--translation
conditioner takes the masked 13-dimensional coefficient vector together with
the 128-dimensional context, uses two width-256 SiLU hidden layers, and has a
zero-initialized output layer. Its log-scale is bounded as
$\widetilde s=2\tanh(s/2)$. The base law is $\mathcal N(0,I_{13})$,
temperature is one, and 128 particles are generated for each input law.

Training minimizes the mean debiased GeomLoss Sinkhorn divergence with $p=2$,
blur $0.01$, scaling $0.5$, automatic backend selection, and no auxiliary mean
loss. AdamW uses learning rate $10^{-3}$ and weight decay $10^{-3}$; the global
gradient norm is clipped to one. Training runs for at most 1,000 epochs.

\subsubsection{Baseline Implementations}
\label{app:duffing-baselines}
The Fixed-Feature MLP operates on the raw 256-point paths without PCA. A
two-hidden-layer width-64 SiLU network encodes each forcing path into 64
features. A DeepSets module with two width-64 SiLU layers maps these features
to 128-dimensional sample embeddings, concatenates their population mean and
standard deviation (with variance stabilization $10^{-6}$), and applies two
more width-64 SiLU layers to produce a 64-dimensional law context. Given this
context and 32-dimensional standard Gaussian noise at temperature one, a
three-hidden-layer width-64 SiLU MLP generates a 32-dimensional output code;
a two-hidden-layer width-64 SiLU decoder maps each code directly to a
256-point response path. It generates 128 paths per law and trains for 200
epochs with batches of eight laws. Its debiased Sinkhorn loss uses $p=2$,
blur $0.05$, scaling $0.5$, and automatic backend selection. AdamW uses
learning rate $10^{-3}$ and weight decay $10^{-2}$, with global gradient norm
clipped to one.

The Nadaraya--Watson baseline computes Sinkhorn divergences $d_k$ between a
query forcing law and each of the 1,000 training laws, then uses
$w_k\propto\exp[-d_k/(2h^2)]$. Its 64 predicted response particles are the
elementwise weighted averages
\begin{equation*}
\widehat Y_j=\sum_{k=1}^{1000}w_kY_{k,j},\qquad j=1,\ldots,64,
\end{equation*}
of the stored target-particle arrays, rather than samples from a discrete
mixture over training laws. The common train--test split and a fixed subset of
64 particles are used for every law. Input-law distances are debiased GeomLoss Sinkhorn
divergences with $p=2$, blur $0.05$, scaling $0.5$, the tensorized backend,
four reference laws per distance chunk, and two query laws per batch. The
bandwidth is fixed at $h=0.005$, and the reported baseline is a single
deterministic fit.

\subsubsection{Evaluation Metrics}
\label{app:duffing-metrics}
For a target empirical law
$\nu=n^{-1}\sum_{i=1}^{n}\delta_{y_i}$ and a predicted empirical law
$\widehat\nu=m^{-1}\sum_{j=1}^{m}\delta_{\widehat y_j}$, let
$\Pi(\widehat\nu,\nu)$ denote their couplings. With ground cost $c$ and
regularization $\varepsilon>0$, define
\begin{align*}
\operatorname{OT}_{\varepsilon}(\widehat\nu,\nu)
&=\inf_{\pi\in\Pi(\widehat\nu,\nu)}
\left\{\int c(x,y)d\pi(x,y)
+\varepsilon\operatorname{KL}(\pi\|\widehat\nu\otimes\nu)\right\},\\
S_{\varepsilon}(\widehat\nu,\nu)
&=\operatorname{OT}_{\varepsilon}(\widehat\nu,\nu)
-\tfrac12\operatorname{OT}_{\varepsilon}(\widehat\nu,\widehat\nu)
-\tfrac12\operatorname{OT}_{\varepsilon}(\nu,\nu).
\end{align*}
For a positive-definite kernel $k$, the biased empirical MMD and empirical
energy distance are
\begin{align*}
\widehat{\operatorname{MMD}}_k
&=\left[\frac1{m^2}\sum_{j,j'}k(\widehat y_j,\widehat y_{j'})
+\frac1{n^2}\sum_{i,i'}k(y_i,y_{i'})
-\frac{2}{mn}\sum_{j,i}k(\widehat y_j,y_i)\right]^{1/2},\\
\widehat{\operatorname{ED}}
&=\frac{2}{mn}\sum_{j,i}\|\widehat y_j-y_i\|_2
-\frac1{m^2}\sum_{j,j'}\|\widehat y_j-\widehat y_{j'}\|_2
-\frac1{n^2}\sum_{i,i'}\|y_i-y_{i'}\|_2.
\end{align*}
For $P_\theta(z)=\langle\theta,z\rangle$ and uniform measure $\sigma$ on
the unit sphere, sliced 2-Wasserstein distance is
\begin{equation*}
\SW_2^2(\widehat\nu,\nu)=\int
W_2^2\!\left((P_\theta)_\#\widehat\nu,
(P_\theta)_\#\nu\right)d\sigma(\theta).
\end{equation*}

The model-specific training losses and schedules are given in
Appendices~\ref{app:duffing-model} and~\ref{app:duffing-baselines}. Final
evaluation generates 128 predicted paths for each target law. We report the
model's configured debiased Sinkhorn divergence, biased RBF
MMD, sliced $W_2$, and empirical energy distance on the reconstructed
256-point paths. For MMD, a separate median-heuristic bandwidth is computed
for each law from all prediction--target squared distances and the nonzero
within-sample squared distances; the reported value is MMD, not MMD squared.
Sliced $W_2$ uses 128 fixed unit-normalized Gaussian projections. Metrics are
averaged over laws. The conditional-flow values
in Table~\ref{tab:duffing} are test-set mean $\pm$ sample standard deviation;
kernel regression is evaluated once.

\section{Supplementary Experiments}
\label{app:supplementary-experiments}

The controlled Gaussian study below is a purely supplementary experiment. Its
complete setup, results, and implementation details are reported here.

\subsection{Controlled Gaussian Distributional Mapping}
\label{sec:controlled-gaussian}

\subsubsection{Problem Setup}

This controlled task learns a map from four-dimensional Gaussian-mixture laws
to four-dimensional Gaussian laws from unpaired ensembles of 200 input and
200 output particles.  The target mean and covariance are nonlinear functions
of analytic input-law moments and Fourier features, while the Distributional
Operator sees only the empirical input law and predicts a full-covariance
Gaussian. Complete data-generation,
baseline, architecture, optimization, and metric definitions are given in
Appendix~\ref{app:controlled-details}.

\subsubsection{Model Instantiation}

The Distributional Operator applies a shared particle encoder
$\phi_\theta:\R^4\to\R^{64}$ to each input sample and averages the resulting
features. The pooled representation is processed by an outer network
$\rho_\theta$ and a Gaussian output head:
\begin{align*}
h^{(b)}&=\frac1{200}\sum_{j=1}^{200}\phi_\theta(x_j^{(b)}),\\
(\widehat\mu^{(b)},\widehat u^{(b)})&=\rho_\theta(h^{(b)}),\\
\widehat L^{(b)}&=\mathcal R_{\mathrm{ltr}}(\widehat u^{(b)}),\qquad
\widehat\Sigma^{(b)}=\widehat L^{(b)}\widehat L^{(b)\top}.
\end{align*}
The particle encoder has two width-64 GELU hidden layers, and the outer network
has three width-128 GELU hidden layers. The Gaussian head outputs four mean
coordinates and ten entries of a lower-triangular Cholesky factor. Softplus
diagonal entries with an offset of $10^{-5}$ guarantee a positive-definite
covariance. Thus the predicted law is
$\mathcal N(\widehat\mu^{(b)},\widehat\Sigma^{(b)})$ and is invariant to the
ordering of the input particles.

\subsubsection{Evaluation}

Training minimizes empirical output-sample NLL. We additionally report the
closed-form Gaussian $W_2$ distance, target-to-prediction KL divergence, and
Hellinger distance, averaged over the 200 test laws. Full metric definitions and
optimization details are in Appendix~\ref{app:controlled-metrics}.

We compare against the Fixed-Feature MLP and distribution-space kernel regression; see
Appendix~\ref{app:controlled-baselines} for their implementations.

\begin{table}[htbp]
\centering
\scriptsize
\resizebox{\columnwidth}{!}{%
\begin{tabular}{lrrrr}
\toprule
Model & NLL & $W_2$ & KL divergence & Hellinger distance\\
\midrule
Distributional Operator & $\boldsymbol{4.016462 \pm 0.002379}$ & $\boldsymbol{0.162049 \pm 0.004262}$ & $\boldsymbol{0.048246 \pm 0.001804}$ & $\boldsymbol{0.099313 \pm 0.002195}$\\
Fixed-Feature MLP & $4.033581 \pm 0.002687$ & $0.187644 \pm 0.003589$ & $0.066297 \pm 0.002362$ & $0.112535 \pm 0.002348$\\
Kernel regression & $4.776576 \pm 0.000000$ & $0.741837 \pm 0.000000$ & $0.806815 \pm 0.000000$ & $0.396992 \pm 0.000000$\\
\bottomrule
\end{tabular}%
}
\caption{Test results for the controlled Gaussian mapping. Entries are the mean
$\pm$ standard deviation over five runs. Kernel regression
is deterministic and therefore has zero
standard deviation.  Lower is better for every metric.}
\label{tab:controlled}
\end{table}

\begin{figure}[htbp]
\centering
\includegraphics[width=0.72\textwidth]{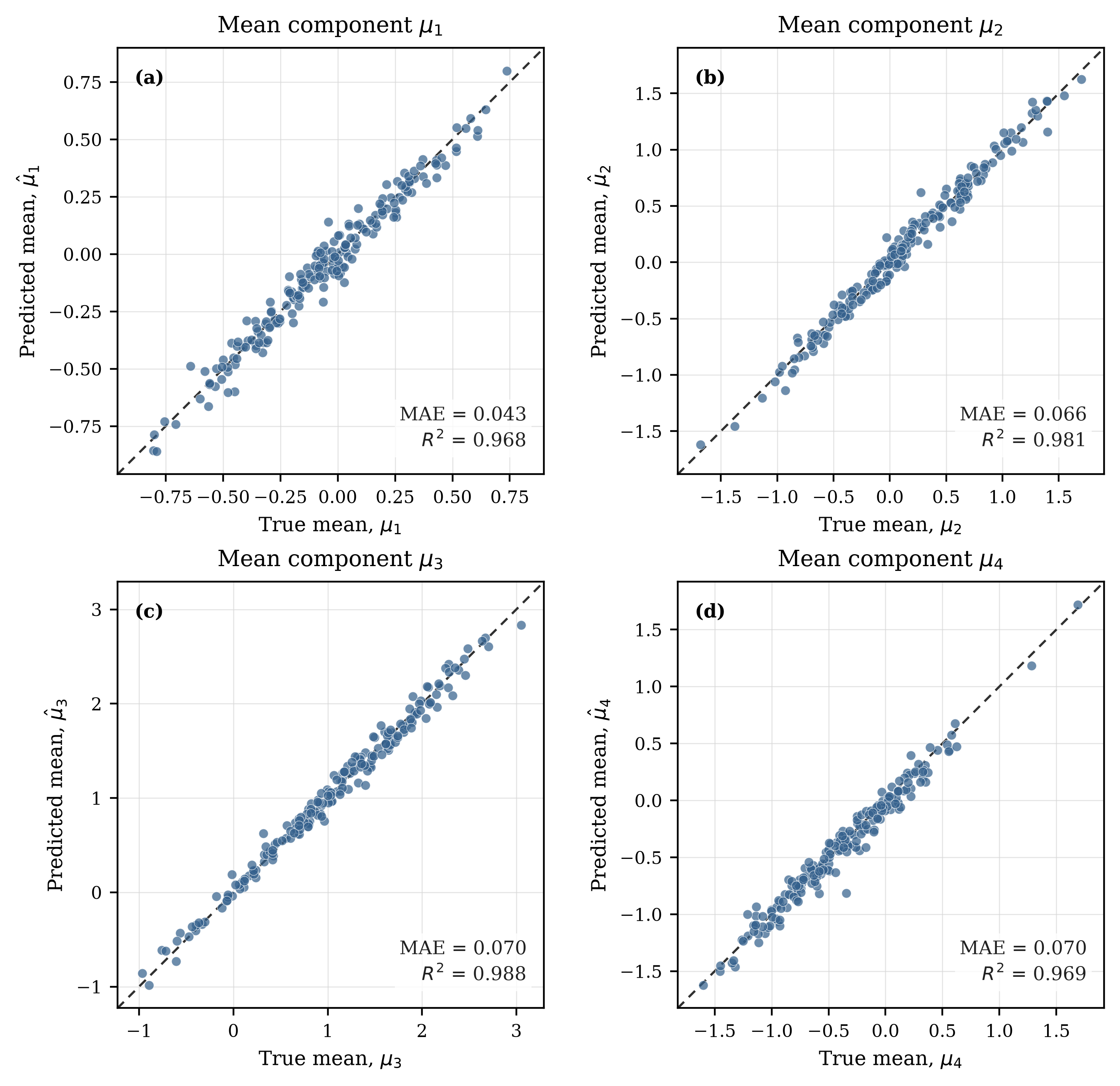}
\caption{Predicted versus true output means on the test set for the trained
Distributional Operator. Each panel
shows one of the four mean components, with one point per test distribution;
the dashed diagonal $y=x$ denotes perfect prediction.}
\label{fig:controlled-mean-scatter}
\end{figure}

\begin{figure}[htbp]
\centering
\includegraphics[width=0.65\textwidth]{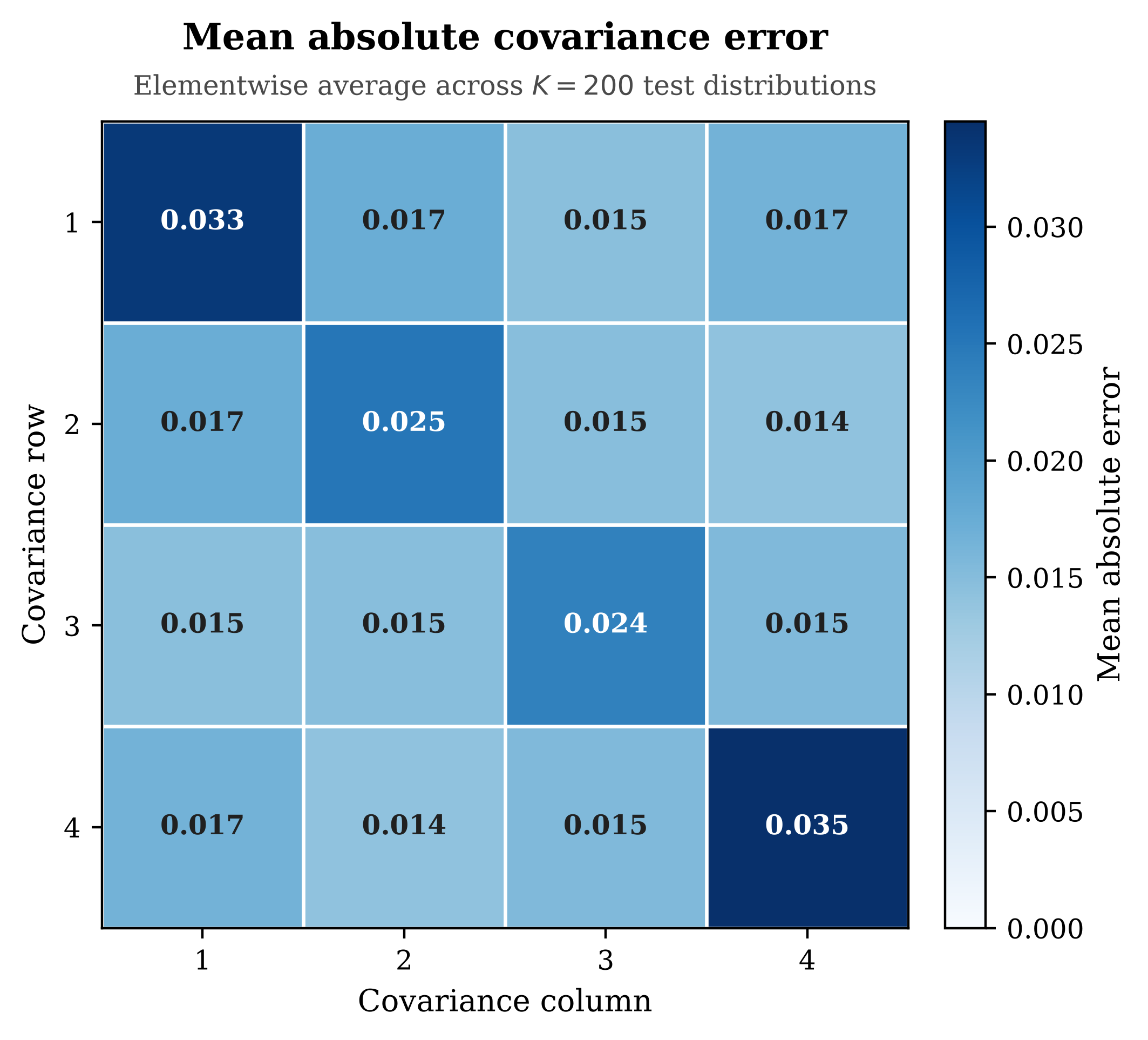}
\caption{Elementwise mean absolute error of the predicted output covariance
matrices on the test set for the trained Distributional Operator. For covariance entry $(i,j)$, the displayed value
is $N_{\mathrm{test}}^{-1}\sum_{k=1}^{N_{\mathrm{test}}}
|\widehat\Sigma_{ij}^{(k)}-\Sigma_{ij}^{(k)}|$; darker errors identify
components that are systematically more difficult to predict.}
\label{fig:controlled-covariance-mae}
\end{figure}

Figure~\ref{fig:controlled-mean-scatter} compares the predicted and true mean
components directly, while Figure~\ref{fig:controlled-covariance-mae}
summarizes covariance error across all test distributions.

\subsubsection{Additional Details}
\label{app:controlled-details}

\paragraph{Data Generation}
\label{app:controlled-data}
For each of the 1,200 laws, the
number of mixture components $C_k$ is sampled uniformly from $\{1,2,3\}$.
Independent unit-rate exponential variables are normalized to give mixture
weights
\begin{equation*}
\pi_{kc}=\frac{u_{kc}}{\sum_{c'=1}^{C_k}u_{kc'}}.
\end{equation*}
Component means and coordinatewise standard deviations are sampled
independently as
\begin{equation*}
m_{kcj}\sim\Unif[-2,2],\qquad s_{kcj}\sim\Unif[0.15,0.80],
\end{equation*}
giving the input law
\begin{equation*}
P_k=\sum_{c=1}^{C_k}\pi_{kc}
\mathcal N\!\left(m_{kc},\diag(s_{kc}^2)\right).
\end{equation*}
We draw $S=200$ independent input particles from each $P_k$.

The deterministic law feature $z_k\in\R^{24}$ concatenates the analytic
mixture mean, elementwise variance, and eight sine and eight cosine Fourier
features. Eight frequency vectors $\omega_r\sim\mathcal N(0,I_4)$ are sampled
once and shared by the dataset. Specifically,
\begin{align*}
\bar m_k&=\sum_c\pi_{kc}m_{kc},\\
v_k&=\sum_c\pi_{kc}(s_{kc}^2+m_{kc}^2)-\bar m_k^2,\\
a_{kr}&=\sum_c\pi_{kc}\exp\!\left(-\frac12\sum_j
s_{kcj}^2\omega_{rj}^2\right)\sin(\omega_r^\top m_{kc}),\\
b_{kr}&=\sum_c\pi_{kc}\exp\!\left(-\frac12\sum_j
s_{kcj}^2\omega_{rj}^2\right)\cos(\omega_r^\top m_{kc}),\\
z_k&=[\bar m_k,v_k,a_k,b_k].
\end{align*}
One fixed random map is shared by all laws. Every entry of $W_\mu,W_d,W_o$
is sampled independently from $\mathcal N(0,1/24)$. The output mean is
$\mu_k=W_\mu z_k$. A lower-triangular matrix $L_k$ is constructed by
\begin{equation*}
(L_k)_{ii}=0.25+0.75\{1+\exp[-(W_dz_k)_i]\}^{-1},
\end{equation*}
and $(L_k)_{ij}=0.2\tanh((W_oz_k)_{ij})$ for $i>j$. Thus
\begin{equation*}
\Sigma_k=L_kL_k^\top+10^{-4}I_4,
\qquad Q_k=\mathcal N(\mu_k,\Sigma_k).
\end{equation*}
Finally, 200 independent output particles are drawn from $Q_k$. The analytic
targets support evaluation but are not used by the principal training loss.

\paragraph{Baseline Implementations}
\label{app:controlled-baselines}
The \emph{Fixed-Feature MLP} computes the empirical mean and unbiased empirical
covariance of each input ensemble, concatenates all four mean and sixteen
covariance entries, and maps the resulting 20 features through four
width-128 GELU layers. The same full-covariance Gaussian head used by the
Distributional Operator returns the predicted mean and covariance. It uses
the same split, objective, optimizer, scheduler, and evaluation procedure as
the main model.

The parameter-free kernel estimator stores empirical input and output means
and covariances for all 1,000 training laws. For a query, it computes Gaussian
$W_2$ distances $d_k$ between its empirical input moments and those of each
training law, then assigns radial-basis weights
\begin{equation*}
\alpha_k=\frac{\exp[-d_k^2/(2h^2)]}
{\sum_{\ell=1}^{1000}\exp[-d_\ell^2/(2h^2)]},\qquad h=1.
\end{equation*}
It interpolates empirical output moments as
\begin{equation*}
\widehat\mu=\sum_k\alpha_k\widetilde m_k^y,
\qquad
\widehat\Sigma=\sum_k\alpha_k\widetilde C_k^y+10^{-5}I_4.
\end{equation*}
The covariance is symmetrized before stabilization and likelihood evaluation.

\paragraph{Evaluation Metrics}
\label{app:controlled-metrics}
A fixed random permutation defines the common train--test split. Five runs use
mini-batches of 64 under the default training configuration.

For target $p=\mathcal N(\mu,\Sigma)$ and prediction
$q=\mathcal N(\widehat\mu,\widehat\Sigma)$, we report empirical output-sample
NLL, Gaussian $W_2$, target-to-prediction KL divergence, and Hellinger
distance. The latter two are
\begin{align*}
D_{\mathrm{KL}}(p\Vert q)
&=\frac12\left[\tr(\widehat\Sigma^{-1}\Sigma)
+(\widehat\mu-\mu)^\top\widehat\Sigma^{-1}(\widehat\mu-\mu)
-4+\log\frac{\det\widehat\Sigma}{\det\Sigma}\right],\\
H(p,q)&=\sqrt{1-\mathrm{BC}(p,q)},\\
\mathrm{BC}(p,q)&=
\frac{(\det\Sigma)^{1/4}(\det\widehat\Sigma)^{1/4}}
{\det((\Sigma+\widehat\Sigma)/2)^{1/2}}
\exp\!\left[-\frac18(\mu-\widehat\mu)^\top
\left(\frac{\Sigma+\widehat\Sigma}{2}\right)^{-1}
(\mu-\widehat\mu)\right].
\end{align*}
Each metric is averaged over the 200 test laws; NLL is additionally
averaged over the 200 observed output particles per law.

\end{document}

%% file: path_law_theory.tex
\subsubsection{Distributional Approximation in Wasserstein Space}
\label{sec:path-law-theory}

From a distributional perspective, a random-field operator does not act
directly on a representation of random variables defined on a fixed
probability space; instead, it acts on probability measures over path spaces.
Following the notation introduced earlier, let
$\mathcal{P}_2(\mathcal{X})$ denote the space of probability measures on
$\mathcal{X}$ with finite second moments, and let $W_2^{\mathcal{X}}$ denote
the Wasserstein distance induced by the norm on $\mathcal{X}$. For a measurable
map $F:\mathcal{X}\to\mathcal{Y}$, write $F_{\#}\mu$ for the pushforward
measure.

In the distributional formulation, a random-field operator is expressed as a
map between probability measures on path spaces. Let $\mathcal{X}$ and
$\mathcal{Y}$ be the path spaces of the input and output random fields,
respectively. We consider a continuous operator of the form
\begin{equation}
  \mathcal{G}:\mathcal{U}\subset\mathcal{P}_2(\mathcal{X})
  \longrightarrow \mathcal{P}_2(\mathcal{Y}),
  \label{eq:path-law-operator}
\end{equation}
where continuity is understood with respect to $W_2^{\mathcal{X}}$ and
$W_2^{\mathcal{Y}}$ on the domain and codomain, respectively. Given a
$W_2^{\mathcal{X}}$-compact set $\mathcal{K}\subset\mathcal{U}$, our goal is
to construct a random-field distributional operator approximator
$\mathcal{N}_{\theta,n,m}$ such that, for every $\varepsilon>0$,
\begin{equation}
  \sup_{\mu\in\mathcal{K}}
  W_2^{\mathcal{Y}}\!\left(
    \mathcal{N}_{\theta,n,m}(\mu),\mathcal{G}(\mu)
  \right)<\varepsilon.
  \label{eq:path-law-approximation-goal}
\end{equation}

The domain is denoted by $\mathcal{U}$ because the subsequent proof must
compare not only $\mathcal{G}(\mu)$ with the output of the approximating
operator, but also $\mathcal{G}(\mu)$ with the output associated with the
input distribution reconstructed by the input autoencoder. The target
operator must therefore be defined and continuous on a set containing both
the original family of input distributions and its family of reconstructed
distributions.

To apply the finite-dimensional distributional operator approximation result
to random-field distributions on path spaces, we use an
encoding--finite-dimensional distributional operator--decoding architecture.
Let
\begin{align}
  \mathcal{E}_{\mathcal{X}}^n &: \mathcal{X}\to\mathbb{R}^n,
  &
  \mathcal{D}_{\mathcal{X}}^n &: \mathbb{R}^n\to\mathcal{X},
  \\
  \mathcal{E}_{\mathcal{Y}}^m &: \mathcal{Y}\to\mathbb{R}^m,
  &
  \mathcal{D}_{\mathcal{Y}}^m &: \mathbb{R}^m\to\mathcal{Y}
\end{align}
be the encoders and decoders for the input and output path spaces,
respectively. Define the corresponding path-space reconstruction operators by
\begin{equation}
  P_{\mathcal{X}}^n
  =\mathcal{D}_{\mathcal{X}}^n\circ\mathcal{E}_{\mathcal{X}}^n,
  \qquad
  P_{\mathcal{Y}}^m
  =\mathcal{D}_{\mathcal{Y}}^m\circ\mathcal{E}_{\mathcal{Y}}^m.
\end{equation}
If $X$ is an $\mathcal{X}$-valued random field with
$\mu=\mathcal{L}(X)$, then the encoded random variable is
$\mathcal{E}_{\mathcal{X}}^n(X)\in\mathbb{R}^n$, whose law is
\begin{equation}
  (\mathcal{E}_{\mathcal{X}}^n)_{\#}\mu
  =\mathcal{L}\!\left(\mathcal{E}_{\mathcal{X}}^n(X)\right)
  \in\mathcal{P}_2(\mathbb{R}^n).
  \label{eq:encoded-path-law}
\end{equation}
Thus, finite-dimensional encoding does not turn the random-field distribution
$\mu$ into a single deterministic vector. Rather, it pushes the path-space
random variable forward to a finite-dimensional random vector while
preserving its probability distribution.

For fixed $n$ and $m$, the finite-dimensional distributional operator
approximation theory developed above provides an abstract
finite-dimensional neural operator module
\begin{equation}
  \widehat{\mathcal{N}}_{\theta}:
  \mathcal{P}_2(\mathbb{R}^n)\longrightarrow
  \mathcal{P}_2(\mathbb{R}^m).
\end{equation}
The random-field distributional operator approximator induced by this module
is defined by
\begin{equation}
  \mathcal{N}_{\theta,n,m}
  = (\mathcal{D}_{\mathcal{Y}}^m)_{\#}
    \circ\widehat{\mathcal{N}}_{\theta}
    \circ(\mathcal{E}_{\mathcal{X}}^n)_{\#}.
  \label{eq:path-law-lifting}
\end{equation}
This architecture can be represented by the following commutative diagram:
\begin{equation}
  \begin{array}{ccc}
    \mathcal{P}_2(\mathcal{X})
      & \xrightarrow{\ \mathcal{G}\ }
      & \mathcal{P}_2(\mathcal{Y})
      \\[0.4em]
    {\scriptstyle(\mathcal{E}_{\mathcal{X}}^n)_{\#}}\Big\downarrow
      &
      & \Big\uparrow{\scriptstyle(\mathcal{D}_{\mathcal{Y}}^m)_{\#}}
      \\[0.4em]
    \mathcal{P}_2(\mathbb{R}^n)
      & \xrightarrow{\ \widehat{\mathcal{N}}_{\theta}\ }
      & \mathcal{P}_2(\mathbb{R}^m).
  \end{array}
  \label{eq:path-law-commutative-diagram}
\end{equation}

In diagram~\eqref{eq:path-law-commutative-diagram}, the upper arrow is the
target random-field distributional operator $\mathcal{G}$; the left arrow
pushes probability measures on the input path space forward to a
finite-dimensional distribution space; the lower arrow is the neural operator
module between finite-dimensional distribution spaces; and the right arrow
decodes the finite-dimensional output distribution back to the output path
space. The central question can therefore be stated as follows: under what
conditions does the approximating path formed by the left, lower, and right
arrows approximate the target operator $\mathcal{G}$ uniformly in
$W_2^{\mathcal{Y}}$?

The proof strategy can be understood by examining the sources of error in the
commutative diagram. First, on the input side, we must control the information
loss caused by finite-dimensional encoding. In other words, after the original
input distribution $\mu$ has been encoded and decoded, the resulting
reconstructed distribution should be close to the original distribution in
$W_2^{\mathcal{X}}$. If this error can be controlled uniformly over the family
of input distributions under consideration, and if the target operator
$\mathcal{G}$ is continuous with respect to Wasserstein distance, then the
finite-dimensional reduction on the input side cannot produce an
uncontrollable error in the output distribution.

Second, once the input has been encoded, the problem becomes one of operator
approximation between finite-dimensional distribution spaces. This part is not
developed again here; it is supplied by the finite-dimensional distributional
operator approximation theory established above and is denoted
uniformly by $\widehat{\mathcal{N}}_{\theta}$. Its role is to approximate the
finite-dimensional distributional map induced by the original operator
$\mathcal{G}$ between the input and output encoding spaces, namely, the lower
arrow in the commutative diagram.

Finally, on the output side, we must control the error introduced when the
finite-dimensional output distribution is decoded back to path space. Even if
the finite-dimensional distributional module gives a good approximation in
$\mathcal{P}_2(\mathbb{R}^m)$, the output decoder must still recover the target
output distribution stably in $W_2^{\mathcal{Y}}$. The output path space must
therefore possess a corresponding uniform autoencoding property in the
Wasserstein sense.

In summary, the approximation error of a random-field distributional operator
can be divided into three components: the input-side autoencoding error, the
finite-dimensional distributional operator approximation error, and the
output-side autoencoding error. The main theorem that follows is organized
precisely around this decomposition. It first controls the reconstruction
errors on the input and output sides through finite-dimensional autoencoding
structures on the path spaces, then invokes the finite-dimensional
distributional operator approximation result established above to
control the error along the lower arrow of the commutative diagram, and
finally obtains uniform Wasserstein approximation of $\mathcal{G}$ by
$\mathcal{N}_{\theta,n,m}$.

It follows that the key to the subsequent theory is not to reconstruct the
finite-dimensional distributional operator, but to prove that the three
errors above can be controlled simultaneously in the Wasserstein sense. This 
formulation does not require the input and output random fields to be paired 
pointwise on the same event space. It is therefore better suited to learning problems in which
random-field distributions are accessed only through sample sets or empirical
distributions. The following theory will make this error decomposition based on
diagram~\eqref{eq:path-law-commutative-diagram} rigorous.